\documentclass{article}
\usepackage{ijcai16}
\usepackage{algorithm}
\usepackage[noend]{algorithmic}
\usepackage{amsmath}
\usepackage{amssymb}
\usepackage{amsthm}
\usepackage{bbm}
\usepackage{bm}
\usepackage{color}
\usepackage{dsfont}
\usepackage{enumerate}
\usepackage{graphicx}
\usepackage{subfigure}
\usepackage{times}

\usepackage[usenames,dvipsnames]{xcolor}

\usepackage[backgroundcolor=White,textwidth=0.5in,disable]{todonotes} %disable comments
\newcommand{\todob}[2][]{\todo[color=Cyan!20,size=\tiny,inline,#1]{B: #2}} % Brano's comments
\newcommand{\todoc}[2][]{\todo[color=Apricot!20,size=\tiny,#1]{Cs: #2}} % Csaba's comments
\newcommand{\todos}[2][]{\todo[color=Yellow!20,size=\tiny,#1]{S: #2}} % Sumeet's comments
\usepackage[capitalize]{cleveref}

\newcommand{\commentout}[1]{}
\newcommand{\junk}[1]{}
\newcommand{\etal}{\emph{et al.}}

\newtheorem{theorem}{Theorem}

\newtheorem{lemma}{Lemma}

\newcommand{\cE}{\mathcal{E}}
\newcommand{\ccE}{\overline{\cE}}
\newcommand{\cF}{\mathcal{F}}
\newcommand{\cH}{\mathcal{H}}

\newcommand{\ceils}[1]{\left\lceil#1\right\rceil}
\newcommand{\condE}[2]{\mathbb{E} \left[#1 \,\middle|\, #2\right]}

\newcommand{\E}[1]{\mathbb{E} \left[#1\right]}

\newcommand{\I}[1]{\mathds{1} \! \left\{#1\right\}}

\newcommand{\rnd}[1]{\mathbf{#1}}
\newcommand{\set}[1]{\left\{#1\right\}}

\DeclareMathOperator*{\argmax}{arg\,max\,}
\DeclareMathOperator*{\argmin}{arg\,min\,}
\mathchardef\mhyphen="2D

\newcommand{\rklucb}{{\tt Rank1ElimKL}}
\newcommand{\bilinucb}{{\tt Rank1Elim}}

\newcommand{\klucb}{{\tt KL\mhyphen UCB}}

\newcommand{\ucb}{{\tt UCB1}}
\newcommand{\ucbelim}{{\tt UCB1Elim}}
\newcommand{\pmx}{p_{\max}}
\title{Bernoulli Rank-$1$ Bandits for Click Feedback}
	
\author{Sumeet Katariya \\ University of Wisconsin-Madison \\ \emph{katariya@wisc.edu}
\And Branislav Kveton \\ Adobe Research \\ \emph{kveton@adobe.com}
\And Csaba Szepesv\'ari \\ University of Alberta \\ \emph{szepesva@cs.ualberta.ca}
\AND Claire Vernade \\ Telecom ParisTech \\ \emph{claire.vernade@telecom-paristech.fr}
\And Zheng Wen \\ Adobe Research \\ \emph{zwen@adobe.com}}
\vspace{-0.05in}

\begin{document}

\maketitle

\begin{abstract}
The probability that a user will click a search result depends both on its relevance and its position on the results page. The \emph{position based model} explains this behavior by ascribing to every item an \emph{attraction} probability, and to every position an \emph{examination} probability. To be clicked, a result must be both attractive and examined. The probabilities of an item-position pair being clicked thus form the entries of a rank-$1$ matrix. We propose the learning problem of a \emph{Bernoulli rank-$1$ bandit} where at each step, the learning agent chooses a pair of row and column arms, and receives the product of their Bernoulli-distributed values as a reward.
This is a special case of the stochastic rank-$1$ bandit problem considered in recent work that proposed
an elimination based algorithm $\bilinucb$, and showed that $\bilinucb$'s regret scales
linearly with the number of rows and columns on ``benign'' instances. These are the instances 
where the minimum of the average row and column
rewards $\mu$ is bounded away from zero. The issue with $\bilinucb$ is that it fails to be competitive 
with straightforward bandit strategies
%that do not take the rank-$1$ structure of the problem in account 
as $\mu \to 0$.
%An ideal algorithm of course should satisfy both of these requirements. 
In this paper we propose $\rklucb$ which simply replaces the (crude) confidence
intervals of $\bilinucb$ with confidence intervals based on Kullback-Leibler (KL) divergences, and with the help of a novel 
result concerning the scaling of KL divergences we prove that with this change,
our algorithm will be competitive no matter the value of $\mu$. 
Experiments with synthetic data confirm that on benign instances
	the performance of $\rklucb$ is significantly better than that of even $\bilinucb$, 
while experiments with models derived from real-data
confirm that the improvements are significant across the board, regardless of
whether the data is benign or not.
\end{abstract}

%!TEX root = Paper.tex

\section{Introduction}
\label{sec:introduction}

When deciding which search results to present, click logs are of particular interest. A fundamental problem in click data is position bias. The probability of an element being clicked depends not only on its relevance, but also on its position on the results page. The position-based model (PBM), first proposed by Richardson \etal~\shortcite{richardson07predicting}, and then formalized by Craswell \etal~\shortcite{craswell08experimental}, models this behavior by associating with each item a probability of being \emph{attractive}, and with each position a probability of being \emph{examined}. To be clicked, a result must be both attractive and examined. Given click logs, the attraction and examination probabilities can be learned using the maximum-likelihood estimation (MLE) or the expectation-maximization (EM) algorithms \cite{chuklin15click}. 

An online learning model for this problem is proposed in Katariya \etal~\shortcite{rank1stochastic}, called \emph{stochastic rank-$1$ bandit}. The objective of the learning agent is to learn the most rewarding item and position, which is the maximum entry of a rank-$1$ matrix. At time $t$, the agent chooses a pair of row and column arms, and receives the product of their values as a reward. The goal of the agent is to maximize its expected cumulative reward, or equivalently to minimize its expected cumulative regret with respect to the optimal solution, the most rewarding pair of row and column arms. This learning problem is challenging because when the agent receives a reward of $0$, it could mean either that the item was unattractive, or the position was left unexamined, or both. 

Katariya \etal~\shortcite{rank1stochastic} also proposed an elimination algorithm, $\bilinucb$, whose regret is $\mathcal{O}( (K+L)\, \mu^{-2}\Delta^{-1} \log n ) $, where $K$ is the number of rows, $L$ is the number of columns, $\Delta$ is the minimum of the row and column gaps, and $\mu$ is the minimum of the average row and column rewards. When $\mu$ is bounded away from zero, the regret scales linearly with $K+L$, while it scales inversely with $\Delta$. This is a significant improvement to using a standard bandit algorithm that (disregarding the problem structure)  would treat item-position pairs as unrelated arms and would achieve a regret of $O(KL\Delta^{-1})$. The issue is that as $\mu$ gets small, the regret bound worsens significantly.
As we verify in \cref{sec:experiments} this indeed happens on models derived from some real-world problems.
% it does happen that most items are not attractive and most positions are not examined, and so the value of $\mu$ can be quite small, in which case the regret of $\bilinucb$ may be overly large.
To illustrate the severity of this problem, 
consider as an example the setting when $K=L$ and the row and column rewards are Bernoulli distributed. 
Let the mean reward of row $1$ and column $1$ be $\Delta$, and the mean reward of all other rows and columns be $0$. We refer to this setting as a `needle in a haystack', because there is a single rewarding entry out of $K^2$ entries. For this setting, $\mu = \Delta/K$, and consequently the regret of $\bilinucb$ is  $\mathcal{O}(\mu^{-2} \Delta^{-1} K\log n) = 
\mathcal{O}(K^3\log n)$. However, a naive bandit algorithm that ignores the rank-$1$ structure and treats each row-column pair as unrelated arms has $\mathcal{O}(K^2 \log n)$ regret.%
\footnote{Alternatively, the worst-case regret bound for $\bilinucb$ becomes $O(K n^{2/3} \log(n))$, while that of for a naive
bandit algorithm with a naive bound is $O(K n^{1/2} \log(n))$.}
 While a naive bandit algorithm is unable to exploit the rank-$1$ structure when $\mu$ is large, $\bilinucb$ is unable to keep up 
 with a naive algorithm when $\mu$ is small. 
% The `needle in a haystack' is an extreme problem, and many real-world problems have more structure. 
Our goal in this paper is to derive an algorithm that performs well across all rank-$1$ problem instances regardless of their parameters.
% What this example illustrates is that $\bilinucb$ is unable to exploit this structure when $\mu$ is small.

%We next illustrate a typical setting where this regret is suboptimal. Consider the setting where the row and column rewards are Bernoulli distributed, and for simplicity assume $K = L$. Assume that the mean reward of row $1$ and column $1$ is $1$, and the mean reward of all other rows and columns is $0$. This setting, where there is a highly attractive item and a highly examined position is very common in real-world settings \cite{chuklin15click}. We refer to problems that can be modeled using this setting as `needle in a haystack', because there is a single rewarding entry out of $K^2$ entries. In this setting, the regret of $\bilinucb$ is $\mathcal{O}(K^3)$. \todob{Why is the regret of $\bilinucb$ this high?} However, a naive bandit algorithm that ignores the rank-$1$ structure and treats each row-column pair as an independent arm has $\mathcal{O}(K^2)$ regret. This shows that either $\bilinucb$ is suboptimal, or its analysis in Katariya \etal~\shortcite{rank1stochastic} is not tight.

In this paper we propose that this improvement can be achieved by replacing the ``$\ucb$ confidence intervals''  used by  $\bilinucb$ 
%that are based on subgaussian tail bounds and that are
 by strictly tighter confidence intervals based on Kullback-Leibler (KL) divergences.
 This leads to our algorithm that we call $\rklucb$.
Based on the work  of Garivier and Cappe \shortcite{garivier11klucb},
we expect this change to lead to an improved behavior, especially, for extreme instances, e.g., as $\mu \to 0$.
Indeed, in this paper we show that KL divergences enjoy a 
peculiar ``scaling'' property, which leads to a significant improvement.
 % between Bernoulli distributions (that are the basis of the $\klucb$ bandit algorithm due to  \cite{garivier11klucb}). 
%This is especially suited for click data modeled using the PBM, because the product of two independent Bernoulli random variables is Bernoulli distributed. 
%The necessary $\bilinucb$ uses $\ucb$ confidence intervals, which are obtained using the Kullback-Leibler divergence between normal distributions. The $\klucb$ confidence intervals, obtained using the Kullback-Leibler divergence between Bernoulli distributions are strictly inside the $\ucb$ confidence intervals, for any bounded distribution \cite{garivier11klucb}. 
%We propose an algorithm, $\rklucb$, that uses $\klucb$ confidence intervals. 
In particular, thanks to this improvement, for the `needle in a haystack' problem discussed above 
the regret of $\rklucb$ becomes $\mathcal{O}(K^2 \log(n))$.

In summary our contributions are as follows:
 First, we propose a \emph{Bernoulli rank-$1$ bandit}, which is a special class of a \emph{stochastic rank-$1$ bandit} where the rewards are Bernoulli distributed. \todoc{I am pretty sure the Bernoulli assumption is not needed.} This has wide applications in click models and we believe that it deserves special attention. Second, we modify $\bilinucb$ for solving the Bernoulli rank-$1$ bandit, which we call $\rklucb$, to use $\klucb$ intervals.
% $\rklucb$ is similar in flavor to $\bilinucb$, where the key idea is to explore all remaining rows and columns randomly over all remaining columns and rows, respectively, to estimate their expected rewards; and then eliminate those rows and columns that seem suboptimal. 
Third, we derive a $O((K+L)\,(\mu\gamma\Delta)^{-1} \log n)$ gap-dependent upper bound on the $n$-step regret of $\rklucb$, where $K,L,\Delta$ and $\mu$ are as above, while $\gamma = \max\set{\mu,1-\pmx}$ with $\pmx$ being the maximum of the row and column rewards; effectively replacing the $\mu^{-2}$ term of the previous regret bound of $\bilinucb$ with $(\mu\gamma)^{-1}$. It follows that the new bound is an unilateral improvement over the previous one and is a strict improvement when $\mu < 1-\pmx$, which is expected to happen quite often in practical problems.
%where we often find that $\mu$ is low (most items are unattractive and most positions are not examined) while $\pmx$ is also expected to be low since in practice we often find that even the best item and position are at best ``moderately good''.
For the `needle in a haystack' problem the new bound essentially matches that of the naive bandit algorithm's bound, while never worsening the bound of $\bilinucb$. 
%This is a significant improvement because $\mu$ is typically very low in click data - most items are not attractive and most positions are not examined. Furthermore, no item is always attractive and no position is always examined, and so typically $\pmx \ll 1$. Hence $1 - \pmx \gg \mu$. 
Our final contribution is the experimental validation of $\rklucb$, on both synthetic and real-world problems.
The experiments indicate that $\rklucb$ outperforms several baselines across almost all problem instances.

%\todob{There is a fundamental problem in how we present our problem of interest. The ``needle in haystack'' problem does not have any structure and $\ucb$ is optimal. Therefore, this is not our problem of interest. Our problems of interest are problems that are ``close'' to this problem.}

We denote random variables by boldface letters and define $[n] = \set{1, \dots, n}$. For any sets $A$ and $B$, we denote by $A^B$ the set of all vectors whose entries are indexed by $B$ and take values from $A$. We let $d(p,q) = p \log \frac{p}{q} + (1-p) \log \frac{1-p}{1-q}$ denote the KL divergence between the Bernoulli distributions with means $p,q\in [0,1]$. As usual, the formula for $d(p,q)$ is defined through its continuous extension as $p,q$ approach the boundaries of $[0,1]$.

%!TEX root = Paper.tex

\section{Setting}
\label{sec:setting}

The setting of the \emph{Bernoulli rank-$1$ bandit} is the same as that of the stochastic rank-$1$ bandit \cite{rank1stochastic}, with the additional requirement that the row and column rewards are Bernoulli distributed. We state the setting for completeness, and borrow the notation from Katariya \etal~\shortcite{rank1stochastic} for the ease of comparison.

An instance of our learning problem is a tuple $B = (K, L, P_\textsc{u}, P_\textsc{v})$, where $K$ is the number of rows, $L$ is the number of columns, $P_{\textsc{u}}$ is a distribution over $\set{0, 1}^K$ from which the row rewards are drawn, and $P_{\textsc{v}}$ is a distribution over $\set{0, 1}^L$ from which the column rewards are drawn. 

Let the row and column rewards be
\begin{align*}
	(\rnd{u}_t,\rnd{v}_t) \stackrel{\text{i.i.d}}\sim P_{\textsc{u}} \otimes P_{\textsc{v}} \,,\qquad t = 1,\dots,n\,.
%    \rnd{u}_t \stackrel{\text{i.i.d}}\sim P_{\textsc{u}} \quad \forall\, t=1,\dots,n, \\
%    \rnd{v}_t \stackrel{\text{i.i.d}}\sim P_{\textsc{v}} \quad \forall\, t=1,\dots,n.
\end{align*}
In particular, $\rnd{u}_t$ and $\rnd{v}_t$ are independent at any time $t$. At time $t$, the learning agent chooses a row index $\rnd{i}_t \in [K]$ and a column index $\rnd{j}_t \in [L]$, and observes $\rnd{u}_t(\rnd{i}_t) \rnd{v}(\rnd{j}_t)$ as its reward. The indices $\rnd{i}_t$ and $\rnd{j}_t$ chosen by the learning agent are allowed to depend only on the history of the agent up to time $t$.

Let the time horizon be $n$. The goal of the agent is to maximize its expected cumulative reward in $n$ steps. This is equivalent to minimizing the \emph{expected cumulative regret} in $n$ steps
\begin{align*}
  R(n) = \E{\sum_{t = 1}^n R(\rnd{i}_t, \rnd{j}_t, \rnd{u}_t, \rnd{v}_t)}\,,
\end{align*}
where $R(\rnd{i}_t, \rnd{j}_t, \rnd{u}_t, \rnd{v}_t) = \rnd{u}_t(i^\ast) \rnd{v}_t(j^\ast) - \rnd{u}_t(\rnd{i}_t) \rnd{v}_t(\rnd{j}_t)$ is the \emph{instantaneous stochastic regret} of the agent at time $t$, and
\begin{align*}
  \textstyle
  (i^\ast, j^\ast) = \argmax_{(i, j) \in [K] \times [L]} \E{\rnd{u}(i) \rnd{v}(j)}
\end{align*}
is the \emph{optimal solution} in hindsight of knowing $P_\textsc{u}$ and $P_\textsc{v}$.

%!TEX root = Paper.tex

\section{$\rklucb$ Algorithm}
\label{sec:algorithm}

\begin{algorithm}[t!]
	\caption{$\rklucb$ for Bernoulli rank-$1$ bandits.}
	\label{alg:main}
	\begin{algorithmic}[1]
		\STATE // Initialization
        \STATE $t \gets 1$, \ $\tilde{\Delta}_0 \gets 1$, $n_{-1} \gets 0$
		\STATE $\rnd{C}^\textsc{u}_0 \gets \set{0}^{K \times L}$, \ 
		$\rnd{C}^\textsc{v}_0 \gets \set{0}^{K \times L}$,
		\STATE $\rnd{h}^\textsc{u}_0 \gets (1, \dots, K)$, \ 
		$\rnd{h}^\textsc{v}_0 \gets (1, \dots, L)$
		\STATE 
		\FORALL{$\ell = 0, 1, \dots$}
		\STATE $n_\ell \gets \ceils{16 \tilde{\Delta}_\ell^{-2} \log n}$
		\STATE $\rnd{I}_\ell \gets \bigcup_{i \in [K]} 
		\set{\rnd{h}^\textsc{u}_\ell(i)}$, \
		$\rnd{J}_\ell \gets \bigcup_{j \in [L]} 
		\set{\rnd{h}^\textsc{v}_\ell(j)}$
		\STATE
		\STATE // Row and column exploration
		\FOR{$n_\ell - n_{\ell - 1}$ times}
		\STATE Choose uniformly at random column $j \in [L]$
		\STATE $j \gets \rnd{h}^\textsc{v}_\ell(j)$
		\FORALL{$i \in \rnd{I}_\ell$}
		\STATE $\rnd{C}^\textsc{u}_\ell(i, j) \gets \rnd{C}^\textsc{u}_\ell(i, 
		j) + \rnd{u}_t(i) \rnd{v}_t(j)$
		\STATE $t \gets t + 1$
		\ENDFOR
		\STATE Choose uniformly at random row $i \in [K]$
		\STATE $i \gets \rnd{h}^\textsc{u}_\ell(i)$
		\FORALL{$j \in \rnd{J}_\ell$}
		\STATE $\rnd{C}^\textsc{v}_\ell(i, j) \gets \rnd{C}^\textsc{v}_\ell(i, 
		j) + \rnd{u}_t(i) \rnd{v}_t(j)$
		\STATE $t \gets t + 1$
		\ENDFOR
		\ENDFOR
		\STATE
		\STATE // UCBs and LCBs on the expected rewards of all remaining rows 
		and columns with divergence constraint
		$\delta_{\ell} \gets \log n + 3 \log \log n$
		\STATE
		\FORALL{$i \in \rnd{I}_\ell$}
		\STATE $\hat{\rnd{u}}_\ell(i) \gets 
		(1 / n_\ell) \sum_{j = 1}^L \rnd{C}^\textsc{u}_\ell(i, j)$
        \STATE $\rnd{U}^\textsc{u}_\ell(i) \gets \argmax_{q\in [\hat{\rnd{u}}_\ell(i),1]}\left\{
		n_\ell d\left(\hat{\rnd{u}}_\ell(i) , q \right) 
        \leq \delta_{\ell} \right\} $
		\STATE $
		\rnd{L}^\textsc{u}_\ell(i) \gets \argmin_{q\in [0,\hat{\rnd{u}}_\ell(i)]}\left\{
		n_\ell d\left(\hat{\rnd{u}}_\ell(i), q 	\right) 		 
		\leq \delta_\ell \right\} $
		\ENDFOR
		\FORALL{$j \in \rnd{J}_\ell$}
		\STATE $\hat{\rnd{v}}_\ell(j) \gets 
		(1 / n_\ell) \sum_{i = 1}^K \rnd{C}^\textsc{v}_\ell(i, j)$
		\STATE $\rnd{U}^\textsc{v}_\ell(j) \gets \argmax_{q\in [\hat{\rnd{v}}_\ell(j),1]}\left\{
		n_\ell d\left(\hat{\rnd{v}}_\ell(j), q \right) 		 
		\leq \delta_\ell\right\} $
		\STATE $\rnd{L}^\textsc{v}_\ell(j) \gets \argmin_{q\in [0,\hat{\rnd{v}}_\ell(j)]}\left\{
		n_\ell d\left(\hat{\rnd{v}}_\ell(j), q	\right) 		 
		\leq \delta_\ell \right\}$
		\ENDFOR
		\STATE
		\STATE // Row and column elimination
		\STATE $\rnd{i}_\ell \gets \argmax_{i \in \rnd{I}_\ell} 
		\rnd{L}^\textsc{u}_\ell(i)$
		\STATE $\rnd{h}^\textsc{u}_{\ell + 1} \gets \rnd{h}^\textsc{u}_\ell$
		\FORALL{$i = 1, \dots, K$}
		\IF{$\rnd{U}^\textsc{u}_\ell(\rnd{h}^\textsc{u}_\ell(i)) \leq 
		\rnd{L}^\textsc{u}_\ell(\rnd{i}_\ell)$}
		\STATE $\rnd{h}^\textsc{u}_{\ell + 1}(i) \gets \rnd{i}_\ell$
		\ENDIF
		\ENDFOR
		\STATE
		\STATE $\rnd{j}_\ell \gets \argmax_{j \in \rnd{J}_\ell} 
		\rnd{L}^\textsc{v}_\ell(j)$
		\STATE $\rnd{h}^\textsc{v}_{\ell + 1} \gets \rnd{h}^\textsc{v}_\ell$
		\FORALL{$j = 1, \dots, L$}
		\IF{$\rnd{U}^\textsc{v}_\ell(\rnd{h}^\textsc{v}_\ell(j)) \leq 
		\rnd{L}^\textsc{v}_\ell(\rnd{j}_\ell)$}
		\STATE $\rnd{h}^\textsc{v}_{\ell + 1}(j) \gets \rnd{j}_\ell$
		\ENDIF
		\ENDFOR
		\STATE
		\STATE $\tilde{\Delta}_{\ell + 1} \gets \tilde{\Delta}_\ell / 2$, \
		$\rnd{C}^\textsc{u}_{\ell + 1} \gets \rnd{C}^\textsc{u}_\ell$, \
		$\rnd{C}^\textsc{v}_{\ell + 1} \gets \rnd{C}^\textsc{v}_\ell$
		\ENDFOR
	\end{algorithmic}
\end{algorithm}

The pseudocode of our algorithm, $\rklucb$, is in \cref{alg:main}. As noted earlier this algorithm is based on $\bilinucb$ \cite{rank1stochastic} with the difference that we replace their confidence intervals with KL-based confidence intervals.  For the reader's benefit, we explain the full algorithm.

$\rklucb$ is an elimination algorithm that operates in stages, where the elimination is conducted with $\klucb$ confidence intervals. The lengths of the stages quadruple from one stage to the next, and the algorithm is designed such that at the end of stage $\ell$, it eliminates with high probability any row and column whose gap scaled by a problem dependent constant %(depending on $\mu,\gamma$, which will be defined formally in \eqref{eq:average reward} and \eqref{eq:gamma}), 
is at least $\tilde{\Delta}_\ell = 2^{- \ell}$. We denote the \emph{remaining rows and columns} in stage $\ell$ by $\rnd{I}_\ell$ and $\rnd{J}_\ell$, respectively.

Every stage has an exploration phase and an exploitation phase. During row-exploration in stage $\ell$ (lines $12$--$16$), every remaining row is played with a randomly chosen remaining column, \todoc{This said remaining column -- not true. AISTATS paper needs to be updated?} \todos{We do play with remaining columns, although not uniformly random , see line 13} and the rewards are added to the table $\rnd{C}^\textsc{u}_\ell \in \mathbb{R}^{K \times L}$. Similarly, during column-exploration in stage $\ell$ (lines $17$--$21$), every remaining column is played with a randomly chosen remaining row, and the rewards are added to the table $\rnd{C}^\textsc{v}_\ell \in \mathbb{R}^{K \times L}$. We play every row (column) with the same random column (row), and separate the row and column reward tables, so that the expected rewards of any two rows (columns) are scaled by the same quantity at the end of any phase. This facilitates comparison between rows (columns) and elimination in the exploitation phase. The distributions used in selecting random columns and rows are such that the row (column) means increase over time.

In the exploitation phase, we construct high-probability $\klucb$ \cite{garivier11klucb} confidence intervals $[\rnd{L}^\textsc{u}_\ell(i), \rnd{U}^\textsc{u}_\ell(i)]$ for row $i \in \rnd{I}_\ell$, and confidence intervals $[\rnd{L}^\textsc{v}_\ell(j), \rnd{U}^\textsc{v}_\ell(j)]$ for column $j \in \rnd{J}_\ell$. As noted earlier, this is where we depart from $\bilinucb$. The elimination uses row $\rnd{i}_\ell$ and column $\rnd{j}_\ell$, where
\begin{align*}
  \rnd{i}_\ell = \argmax\limits_{i \in \rnd{I}_\ell} \rnd{L}^\textsc{u}_\ell(i)\,, \qquad
  \rnd{j}_\ell = \argmax\limits_{j \in \rnd{J}_\ell} \rnd{L}^\textsc{v}_\ell(j)\,.
\end{align*}
We eliminate any row $i$ and column $j$ such that 
\begin{align*}
  \rnd{U}^\textsc{u}_\ell(i) \le \rnd{L}^\textsc{u}_\ell(\rnd{i}_\ell)\,, \qquad
  \rnd{U}^\textsc{v}_\ell(j) \le \rnd{L}^\textsc{v}_\ell(\rnd{j}_\ell)\,.
\end{align*}
We also track the remaining rows and columns in stage $\ell$ by $\rnd{h}^\textsc{u}_\ell$ and $\rnd{h}^\textsc{v}_\ell$, respectively. When row $i$ is eliminated by row $\rnd{i}_\ell$, we set $\rnd{h}^\textsc{u}_\ell(i) = \rnd{i}_\ell$. If row $\rnd{i}_\ell$ is eliminated by row $\rnd{i}_{\ell'}$ at a later stage $\ell' > \ell$, we update $\rnd{h}^\textsc{u}_\ell(i) = \rnd{i}_{\ell'}$. This is analogous for columns. The remaining rows $\rnd{I}_\ell$ and columns $\rnd{J}_\ell$ can be then defined as the unique values in $\rnd{h}^\textsc{u}_\ell$ and $\rnd{h}^\textsc{v}_\ell$, respectively. The maps $\rnd{h}^{\textsc{u}}_\ell$ and  $\rnd{h}^{\textsc{v}}_\ell$ help to guarantee that the row and column means are nondecreasing.

The $\klucb$ confidence intervals in $\rklucb$ can be found by solving a one-dimensional convex optimization problem for every row (lines $27$--$28$) and column (lines $31$--$32$). They can be found efficiently using binary search because the Kullback-Leibler divergence $d(x, q)$ is convex in $q$ as $q$ moves away from $x$ in either direction. The $\klucb$ confidence intervals need to be computed only once per stage. Hence, $\rklucb$ has to solve at most $K + L$ convex optimization problems per stage, and hence $(K + L )\log n$ problems overall.

%!TEX root = Paper.tex

\section{Analysis}
\label{sec:analysis}

In this section, we derive a gap-dependent upper bound on the $n$-step regret of $\rklucb$. 
The hardness of our learning problem is measured by two kinds of metrics. The first kind are gaps. The \emph{gaps} of row $i \in [K]$ and column $j \in [L]$ are defined as
\begin{align}
  \Delta^\textsc{u}_i = \bar{u}(i^\ast) - \bar{u}(i)\,, \quad
  \Delta^\textsc{v}_j = \bar{v}(j^\ast) - \bar{v}(j)\,,
  \label{eq:gaps}
\end{align}
respectively; and the \emph{minimum row and column gaps} are defined as
\begin{align}
  \Delta^\textsc{u}_{\min} = \!\!\! \min_{i \in [K]: \Delta^\textsc{u}_i > 0} \Delta^\textsc{u}_i\,, \quad
  \Delta^\textsc{v}_{\min} = \!\!\! \min_{j \in [L]: \Delta^\textsc{v}_j > 0} \Delta^\textsc{v}_j\,,
  \label{eq:minimum gaps}
\end{align}
respectively. Roughly speaking, the smaller the gaps, the harder the problem. This inverse dependence on gaps is tight \cite{rank1stochastic}.

The second kind of quantities are the extremal parameters
\begin{align}
  \mu
  & = \min \set{\frac{1}{K} \sum_{i = 1}^K \bar{u}(i), \ \frac{1}{L} \sum_{j = 1}^L \bar{v}(j)}\,,
  \label{eq:average reward} \\
  \pmx
  & = \max \set{\max_{i \in [K]} \bar{u}(i), \ \max_{j \in [L]} \bar{v}(j)}\,.
  \label{eq:maximum reward}
\end{align}
The first metric, $\mu$, is the minimum of the average of entries of $\bar{u}$ and $\bar{v}$. This quantity appears in our analysis due to the averaging character of $\rklucb$. The smaller the value of $\mu$, the larger the regret. The second metric, $\pmx$, is the maximum entry in $\bar{u}$ and $\bar{v}$. As we shall see the regret scales inversely with
\begin{align}
    \gamma = \max\set{\mu,1-\pmx}\,.
\label{eq:gamma}
\end{align}
Note that if $\mu\to 0$ and $\pmx\to 1$ at the same time then the row and columns gaps must also approach one.
%The larger the value of $\pmx$, the larger the regret. The factor of $\pmx$ is due to the lower bound on the scaled KL divergence in \cref{lem:KL scaling}.

With this we are ready to state our main result:
%Our main result is as follows: %upper bound on the regret of $\rklucb$ is stated and proved below.

\begin{theorem}
\label{thm:upper bound} Let $C = 6 e + 82$, $n\ge 5$. The expected $n$-step regret of $\rklucb$ is bounded as
\begin{align*}
  R(n)
  \leq {} & \frac{160}{\mu \gamma} \left(\sum_{i = 1}^K \frac{1}{\bar{\Delta}^\textsc{u}_i} +
  \sum_{j = 1}^L \frac{1}{\bar{\Delta}^\textsc{v}_j}\right) \log n + % {} \\
  %& 
  C(K + L) \,,
\end{align*}
where
\begin{align*}
  \bar{\Delta}^\textsc{u}_i & = \Delta^\textsc{u}_i + \I{\Delta^\textsc{u}_i = 0} \Delta^\textsc{v}_{\min}\,, \\
  \bar{\Delta}^\textsc{v}_j & = \Delta^\textsc{v}_j + \I{\Delta^\textsc{v}_j = 0} \Delta^\textsc{u}_{\min}\,.
\end{align*}
\end{theorem}
The difference from the main result of Katariya \etal~\shortcite{rank1stochastic} is that the first term in our bound scales with $1/(\mu\gamma)$ instead of scaling with $1/\mu^2$. Since $\mu \le \gamma$ and in fact often $\mu \ll \gamma$, this is a significant improvement. For an empirical validation of this, see the next section.
%which, as was noted earlier, will also be confirmed empirically.

Due to the lack of space we only provide a sketch of the proof of \cref{thm:upper bound}, which, at a high level, follows the steps  of the proof of the main result of Katariya \etal~\shortcite{rank1stochastic}. Focusing on the source of the improvement, we first state and prove a new lemma, which, as we shall see, will allow us to replace one of the $1/\mu$ factors with $1/\gamma$ in the regret bound.
Recall from \cref{sec:introduction} that $d$ denotes the KL divergence 
%\begin{align*}
%  d(p, q) = p \log \frac{p}{q} + (1 - p) \log \frac{1 - p}{1 - q}
%\end{align*}
%be the KL divergence 
between Bernoulli random variables with means $p ,q\in [0, 1]$.
\begin{lemma}
\label{lem:KL scaling} Let $c,p,q \in [0, 1]$. Then
\begin{align}
  c (1 - \max \set{p, q}) d(p, q) \leq
  d(c p, c q) \leq
  c d(p, q)\,.
  \label{eq:kl scaling}
\end{align}
and in particular 
\begin{align}
  2 c\max(c,1-\max\set{p,q}) (p-q)^2 \leq d(cp,cq)\,.
  \label{eq:klscaling2}
\end{align}
\end{lemma}
\begin{proof}
The proof of \eqref{eq:kl scaling} is based on differentiation. The first two derivatives of $d(c p, c q)$ with respect to $q$ are
\begin{align*}
  \frac{\partial}{\partial q} d(c p, c q) & = \frac{c (q - p)}{q (1 - c q)}\,, \\
  \frac{\partial^2}{\partial q^2} d(c p, c q) & = \frac{c^2 (q - p)^2 + c p (1 - c p)}{q^2 (1 - c q)^2}\,;
\end{align*}
and the first two derivatives of $c d(p, q)$ with respect to $q$ are
\begin{align*}
  \frac{\partial}{\partial q} [c d(p, q)] & = \frac{c (q - p)}{q (1 - q)}\,, \\
  \frac{\partial^2}{\partial q^2} [c d(p, q)] & = \frac{c (q - p)^2 + c p (1 - p)}{q^2 (1 - q)^2}\,.
\end{align*}
The second derivatives show that both $d(c p, c q)$ and $c d(p, q)$ are convex in $q$ for any $p$. The minima are at $q = p$.

We fix $p$ and $c$, and prove \eqref{eq:kl scaling} for any $q$. The upper bound is derived as follows. Since
\begin{align*}
  d(c p, c x) =
  c d(p, x) =
  0
\end{align*}
when $x = p$, the upper bound holds if $c d(p, x)$ increases faster than $d(c p, c x)$ for any $p < x \leq q$, and if $c d(p, x)$ decreases faster than $d(c p, c x)$ for any $q \leq x < p$. This follows from the definitions of $\frac{\partial}{\partial x} d(c p, c x)$ and $\frac{\partial}{\partial x} [c d(p, x)]$. In particular, both derivatives have the same sign for any $x$, and $1 / (1 - c x) \leq 1 / (1 - x)$ for $x \in [\min \set{p, q}, \max \set{p, q}]$.

The lower bound is derived as follows. Note that the ratio of $\frac{\partial}{\partial x} [c d(p, x)]$ and $\frac{\partial}{\partial x} d(c p, c x)$ is bounded from above as
\begin{align*}
  \frac{\frac{\partial}{\partial x} [c d(p, x)]}{\frac{\partial}{\partial x} d(c p, c x)} =
  \frac{1 - c x}{1 - x} \leq
  \frac{1}{1 - x} \leq
  \frac{1}{1 - \max \set{p, q}}
\end{align*}
for any $x \in [\min \set{p, q}, \max \set{p, q}]$. 
%Reordering and integrating both sides (in $x$) from $p$ to $q$ we get the desired result.
Therefore, we get a lower bound on $d(c p, c q)$ when we multiply $c d(p, q)$ by $1 - \max \set{p, q}$.

To prove \eqref{eq:klscaling2} note that by Pinsker's inequality, for any $p,q$, $d(p,q) \ge 2 (p-q)^2$.
Hence, on the one hand, $d(cp,cq) \ge 2 c^2(p-q)^2$, while on the other hand, from \eqref{eq:kl scaling} we find
that $d(cp,cq) \ge 2 c (1-\max\set{p,q}) (p-q)^2$. Taking the maximum of the right-hand sides gives \eqref{eq:klscaling2}.
\end{proof}

\begin{proof}[Proof sketch of \cref{thm:upper bound}]
%The formal proof is technical and based on similar arguments to those in the proof of $\klucb$ \cite{garivier11klucb}.

We proceed along the lines of Katariya \etal~\shortcite{rank1stochastic}. The key step in their analysis is the upper bound on the expected $n$-step regret of any suboptimal row $i \in [K]$. This bound is proved as follows. First, Katariya \etal~\shortcite{rank1stochastic} show that row $i$ is eliminated with  high probability after $O((\mu \Delta^\textsc{u}_i)^{-2} \log n)$ observations, for any column elimination strategy. Then they argue that the amortized per-observation regret before the elimination is $O(\Delta^\textsc{u}_i)$. Therefore, the total regret of row $i$ is $O(\mu^{-2} (\Delta^\textsc{u}_i)^{-1} \log n)$. The expected $n$-step regret of any suboptimal column $j \in [L]$ is bounded analogously.

We modify the above argument as follows. Roughly speaking, due to the $\klucb$ confidence interval, a suboptimal row $i$ is eliminated with a high probability after
\begin{align*}
  O\left(\frac{1}{d(\mu (\bar{u}(i^\ast) - \Delta^\textsc{u}_i), \mu \bar{u}(i^\ast))} \log n\right)
\end{align*}
observations. Therefore, the expected $n$-step regret of coming from experimenting with row $i$ is
\begin{align*}
  O\left(\frac{\Delta^\textsc{u}_i}{d(\mu (\bar{u}(i^\ast) - \Delta^\textsc{u}_i), \mu \bar{u}(i^\ast))} \log n\right)\,.
\end{align*}
Now we apply \eqref{eq:klscaling2} of \cref{lem:KL scaling} to get that the regret is
\begin{align*}
   O\left(\frac{\Delta^\textsc{u}_i}{d(\mu (\bar{u}(i^\ast) - \Delta^\textsc{u}_i), \mu \bar{u}(i^\ast))} \log n\right) 
%  & \quad = O\left(\frac{\Delta^\textsc{u}_i}{\mu \gamma
%  d(\bar{u}(i^\ast) - \Delta^\textsc{u}_i, \bar{u}(i^\ast))} \log n\right) \\
  = O\left(\frac{1}{\mu \gamma \Delta^\textsc{u}_i} \log n\right)\,.
\end{align*}
The regret of any suboptimal column $j \in [L]$ is bounded analogously.
\end{proof}

%\input{Discussion}

%!TEX root = Paper.tex

\section{Experiments}
\label{sec:experiments}

We conduct two experiments. In \cref{sec:simulated data}, we compare our algorithm to other algorithms available in the literature on a synthetic problem. In \cref{sec:realworld}, we evaluate the same set of algorithms on models built based on a real-world dataset.

\subsection{Rank1Elim, UCB1Elim, and UCB1}
\label{sec:simulated data}

Following Katariya \etal~\shortcite{rank1stochastic}, we consider the `needle in a haystack' class of problems, where only one item is attractive and one position is examined.  We recall the problem here. The $i$-th entry of $\rnd{u}_t$, $\rnd{u}_t(i)$, and the $j$-th entry of $\rnd{v}_t$, $\rnd{v}_t(j)$, are independent Bernoulli variables with mean
\begin{align}
    \begin{split}
        \bar{u}(i) &= p_\textsc{u} + \Delta_\textsc{u} \I{i = 1},\\
        \bar{v}(j) &= p_\textsc{v} + \Delta_\textsc{v} \I{j = 1},
    \end{split}
    \label{eq:discussion problem}
\end{align}
for some $(p_\textsc{u}, p_\textsc{v}) \in [0, 1]^2$ and gaps $(\Delta_\textsc{u},\Delta_\textsc{v}) \in (0, 1 - p_\textsc{u}]\times (0, 1 - p_\textsc{v}]$. Note that arm $(1,1)$ is optimal with an expected reward of $(p_\textsc{u} + \Delta_\textsc{u})(p_\textsc{v} + \Delta_\textsc{v})$.

\begin{figure*}[t]
  \centering
  \includegraphics[width=1.8in]{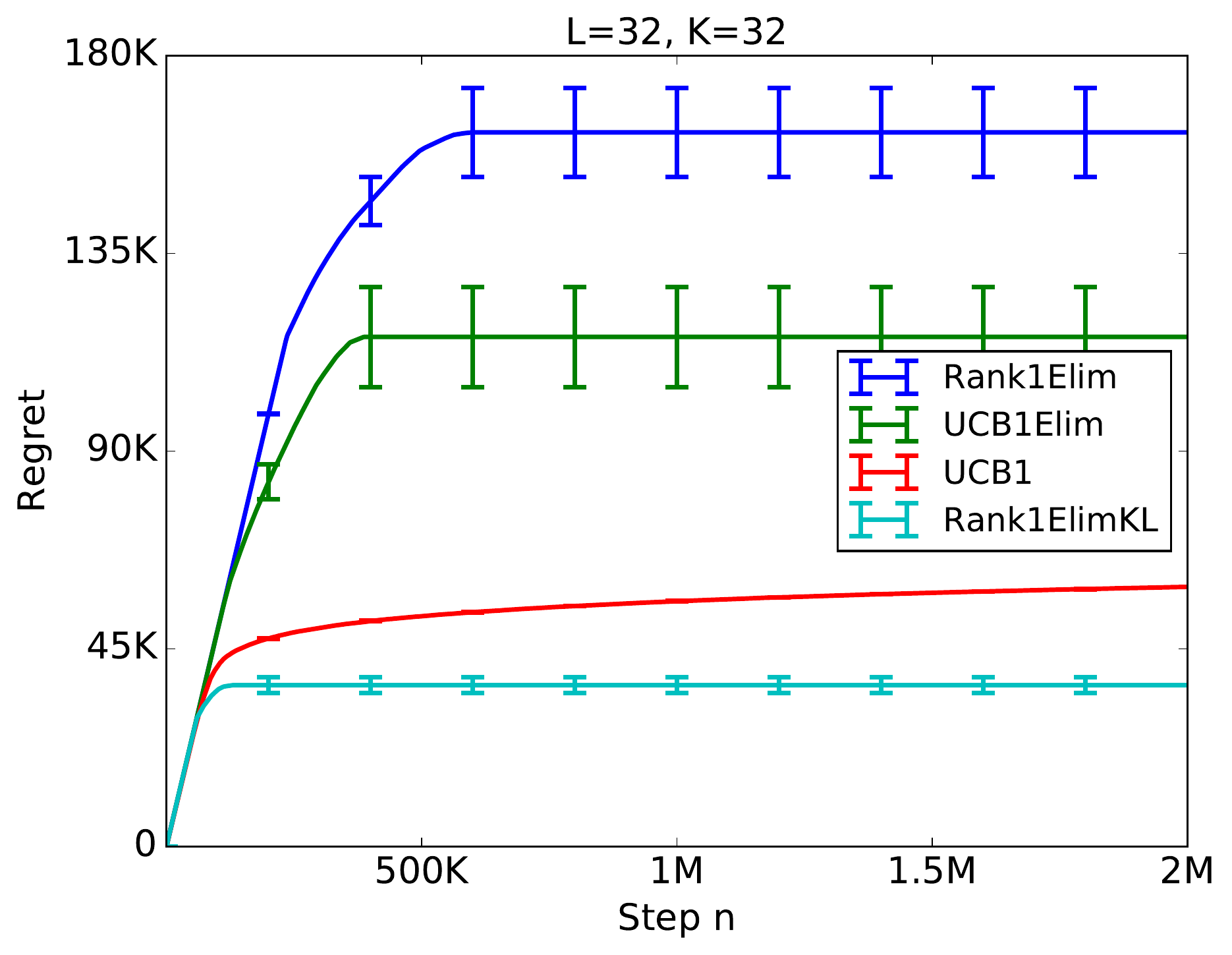}
  \includegraphics[width=1.8in]{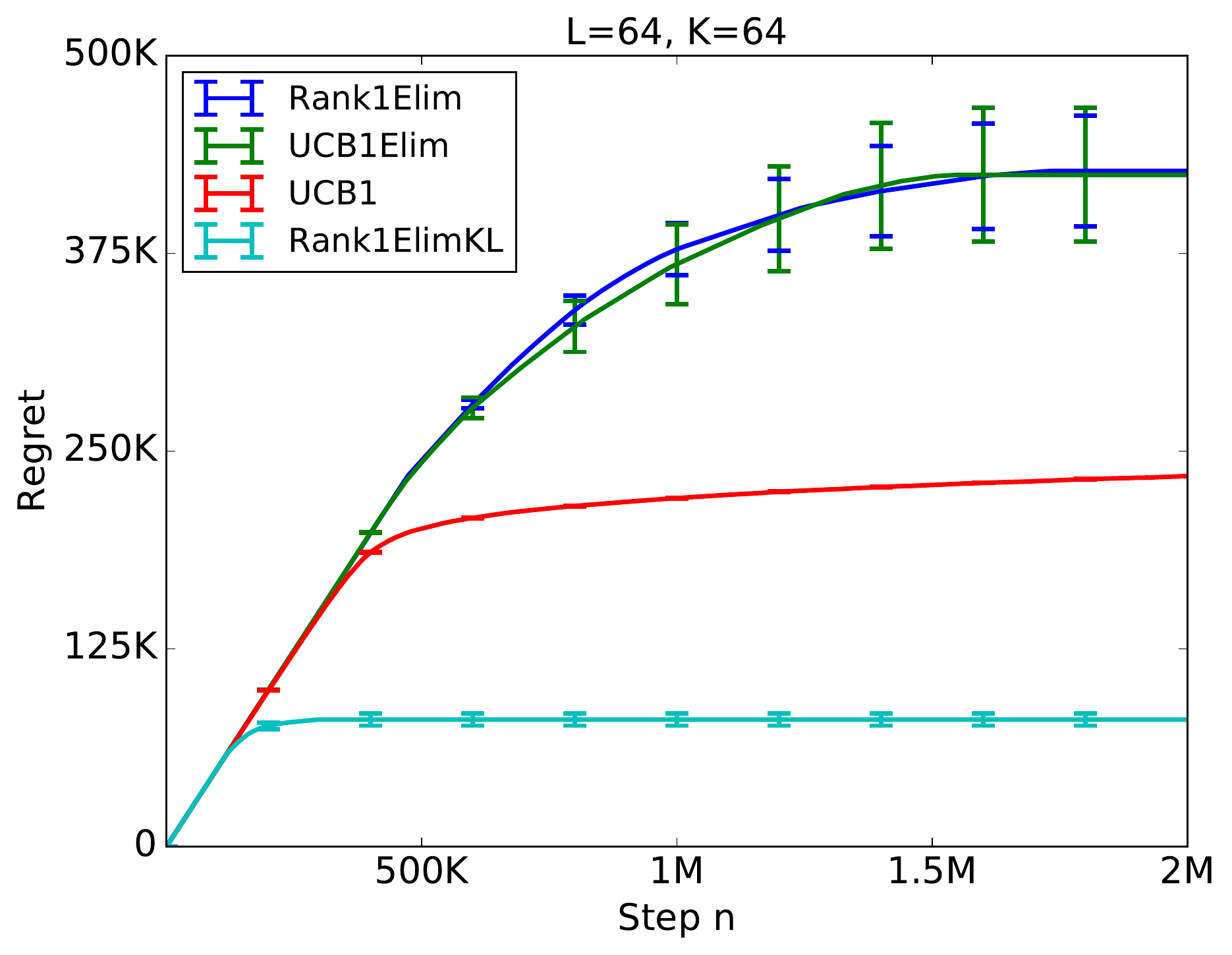}
  \includegraphics[width=1.8in]{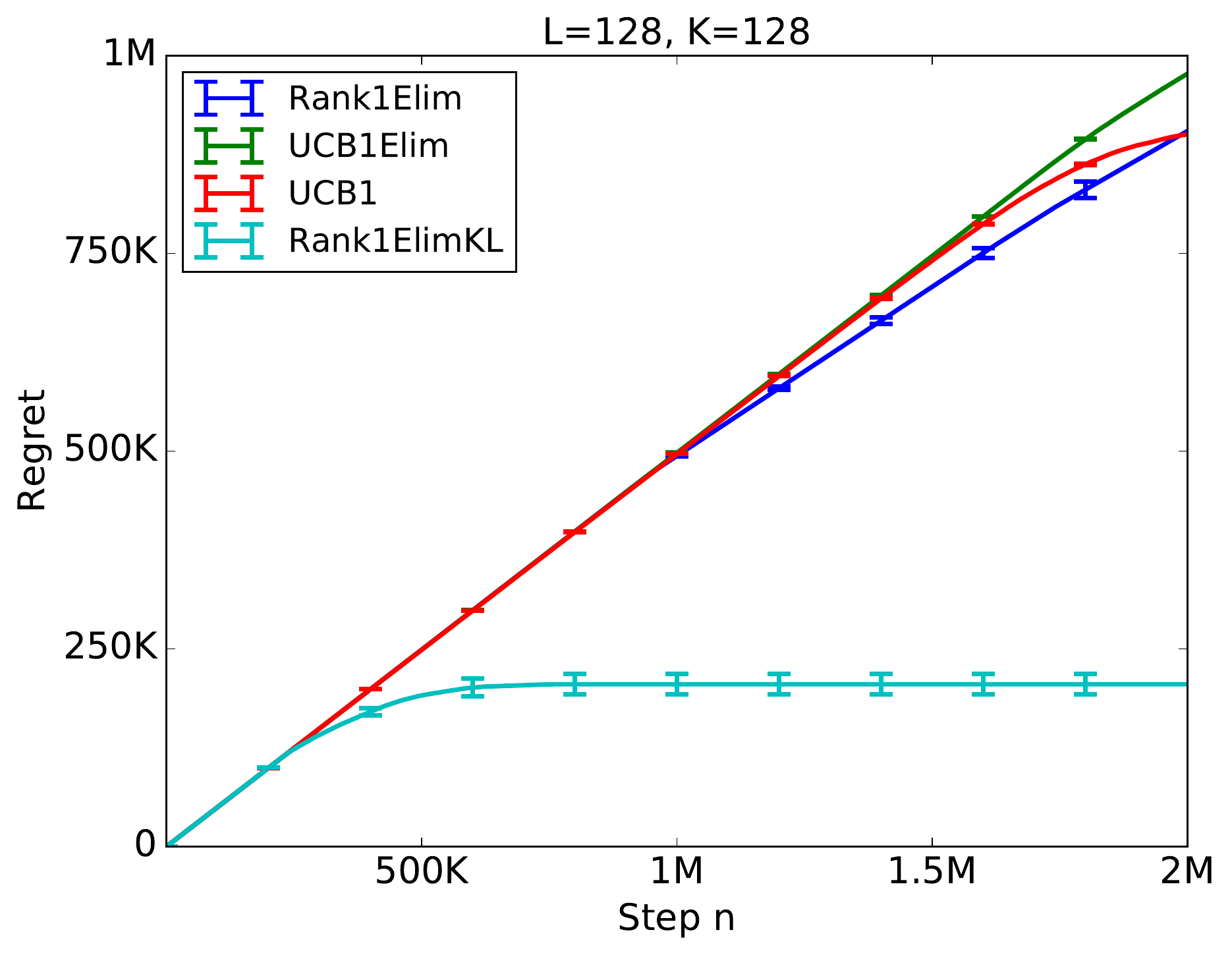} %\vspace{-0.05in}
  \hspace{3.1in} (a) \hspace{1.6in} (b) \hspace{1.6in} (c) \vspace{-0.05in}
  \caption{The $n$-step regret of $\rklucb$, $\ucbelim$, $\bilinucb$ and $\ucb$ on the problem \eqref{eq:discussion problem} for \textbf{a.} $K = L = 32$ \textbf{b.} $K = L = 64$ \textbf{c.} $K = L = 128$. The results are averaged over $20$ runs.}
  \label{fig:comparison}
\end{figure*}

\begin{figure*}[h]
  \centering
  \includegraphics[width=1.72in]{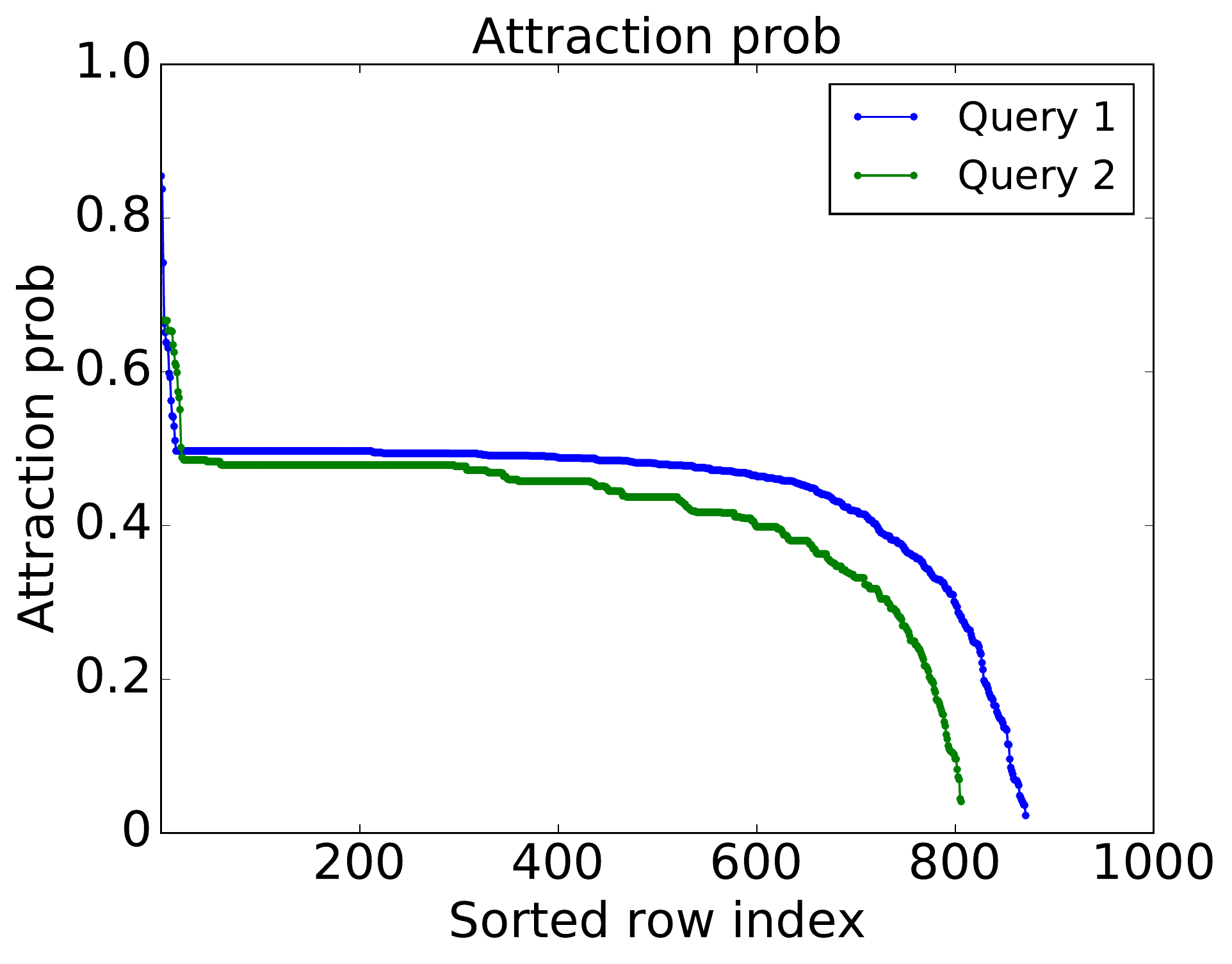}
  \includegraphics[width=1.67in]{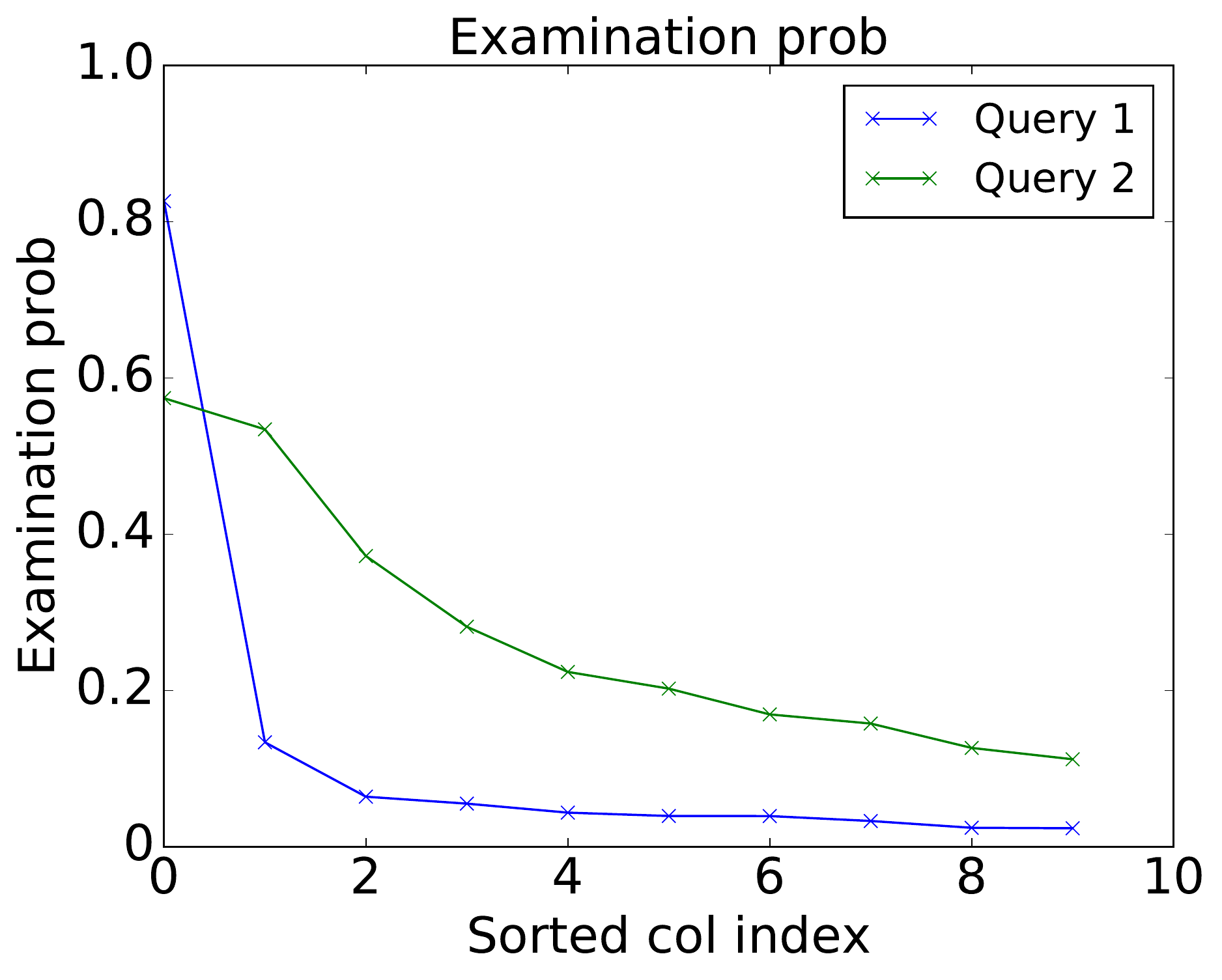}
  \includegraphics[width=1.71in]{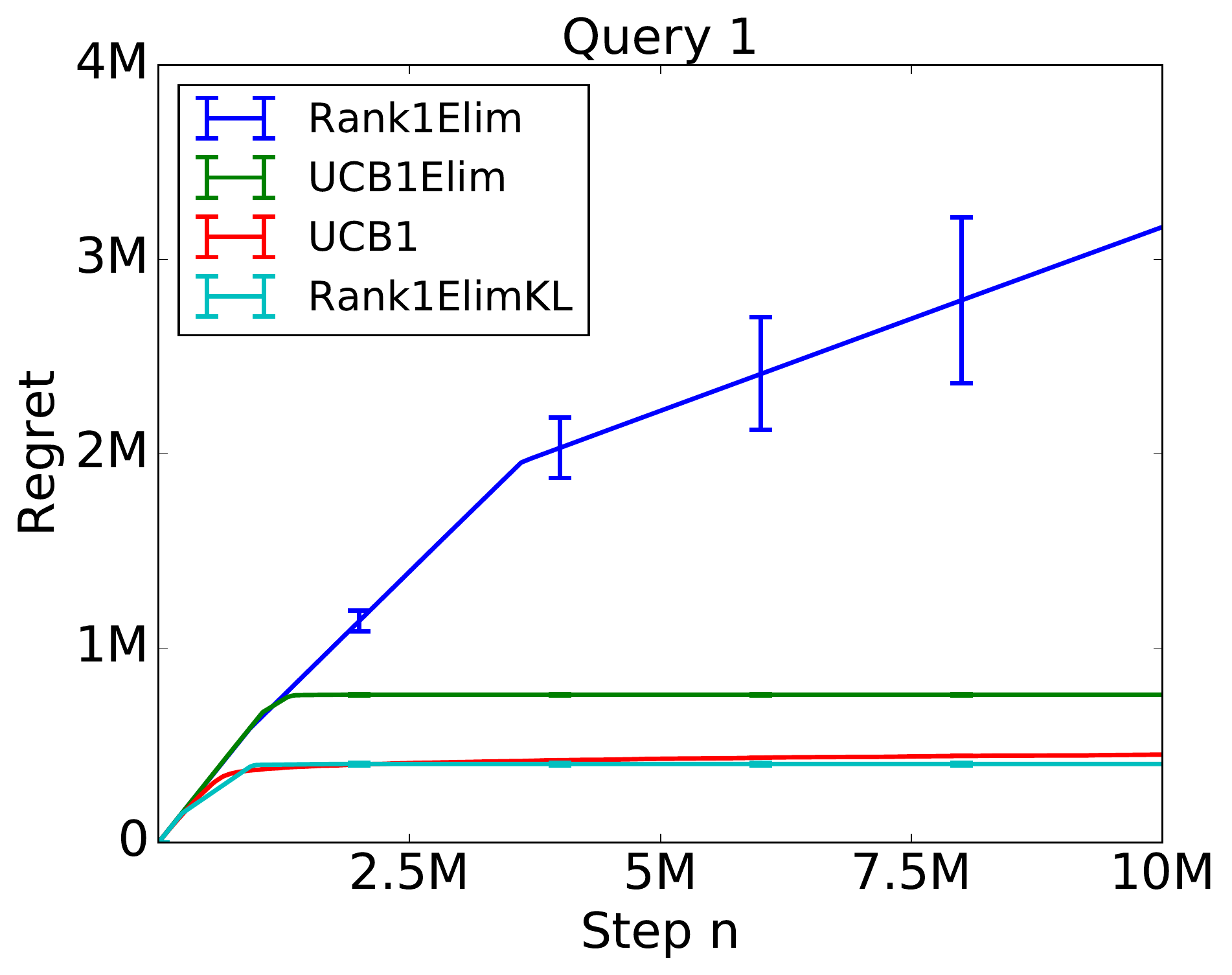} 
  \includegraphics[width=1.79in]{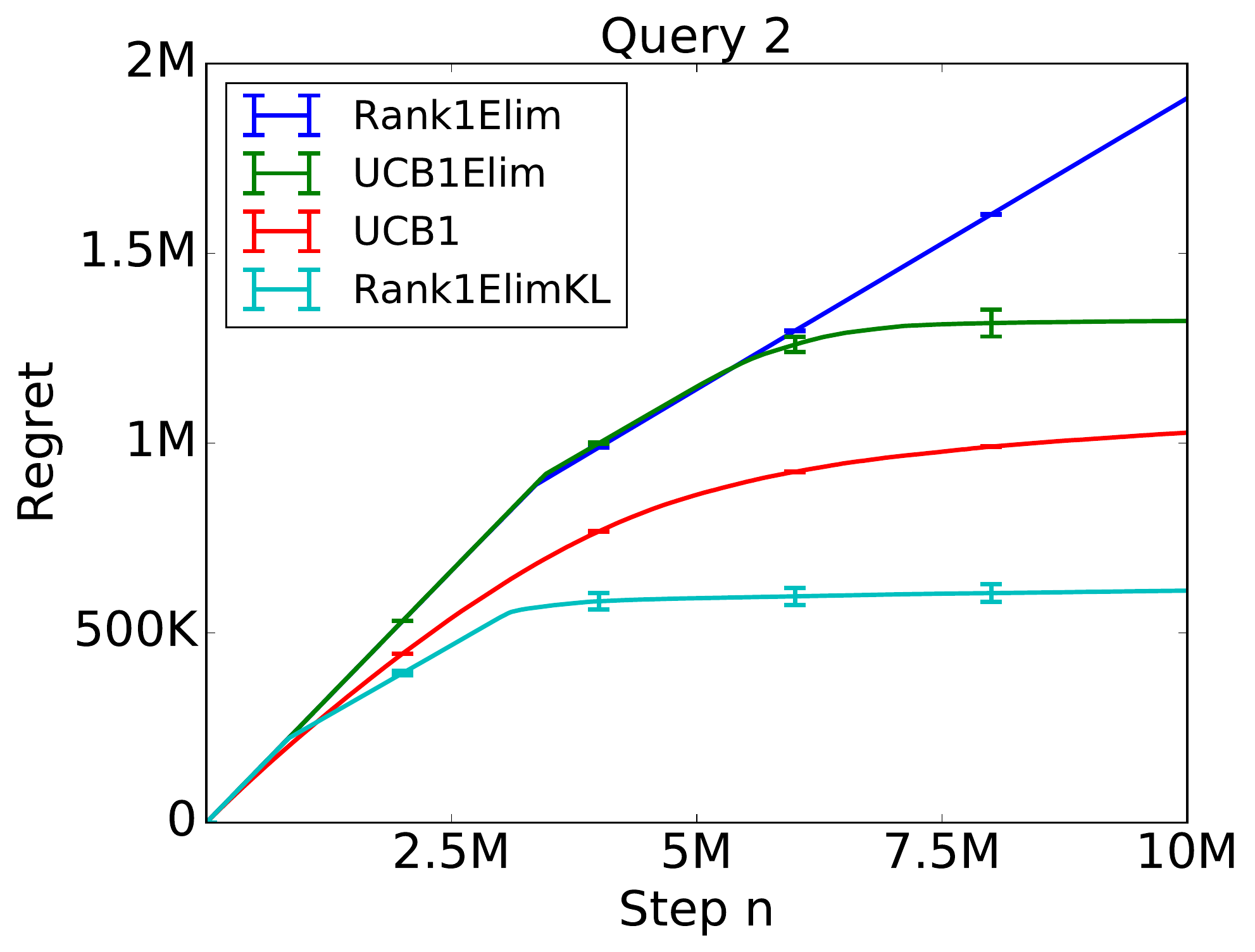} \vspace{-0.05in}
  \hspace{0.2in} (a) \hspace{1.5in} (b) \hspace{1.5in} (c) \hspace{1.5in} (d) \vspace{-0.05in}
  \caption{\textbf{a.} The sorted attraction probabilities of the items from $2$ queries from the Yandex dataset. \textbf{b. } The sorted examination probabilities of the positions for the same $2$ queries. \textbf{c. }The $n$-step regret for Query $1$. \textbf{d. }Regret for Query $2$. The results are averaged over $5$ runs.}
  \label{fig:samplequeries}
\end{figure*}

The goal of this experiment is to compare $\rklucb$ with three other algorithms from the literature and validate that its regret scales linearly with $K$ and $L$, which implies that it exploits the problem structure. 
In this experiment, we set $p_\textsc{u} = p_\textsc{v} = 0.25$ and $\Delta_\textsc{u} = \Delta_\textsc{v} = 0.5$ so that 
$\mu = (1-1/K) 0.25 + 0.75/K = 0.25 +0.5/K$, $1-\pmx = 0.25$ and $\gamma =\mu =  0.25+0.5/K$. 
\todoc{Note that $\mu = \gamma$!}
% Hence, the difference if any stem

In addition to comparing to $\bilinucb$, we also compare to
%$\bilinucb$ is the algorithm proposed in Katariya \etal~\shortcite{rank1stochastic} for the \emph{stochastic rank-$1$ bandit}, which is an elimination based algorithm similar to $\rklucb$ that uses $\ucb$ confidence intervals. 
$\ucbelim$ \cite{auer10ucb} and $\ucb$ \cite{auer02finitetime}.
$\ucb$ is chosen as a baseline as it has been used by Katariya \etal~\shortcite{rank1stochastic} in their experiments, too, while 
$\ucbelim$ is chosen as it is based on a similar elimination approach as $\bilinucb$ and $\rklucb$.
We opted not to compare to $\klucb$ as we expect it to perform similarly to $\ucb$ as the problem parameters are relatively close to $0.5$.
%As it is well known, $\ucb$ \cite{auer02finitetime} pulls the arm with the highest upper confidence bounds, while
%$\ucbelim$ \cite{auer10ucb} is an elimination based algorithm that eliminates suboptimal arms using confidence intervals that are also derived from Hoeffding's inequality.

\cref{fig:comparison} shows the $n$-step regret of $\rklucb$, $\bilinucb$, $\ucbelim$, and $\ucb$ as a function of time ($n$) for values of $K = L$, the latter of which double from one plot to the next. We observe that only the regret of $\rklucb$ flattens in all three problems. % which indicates that $\rklucb$ learns the optimal arm. 
We also see that the regret of $\rklucb$ doubles as $K$ and $L$ double, indicating that our bound in \cref{thm:upper bound} has the right scaling in $K + L$, and that the algorithm leverages the problem structure. On the other hand, the regret of $\ucb$ and $\ucbelim$ quadruples when $K$ and $L$ double, because their regret is $\Omega(KL)$. Finally, in all problems, we observe that $\rklucb$ outperforms all other algorithms, which indicates that it leverages the structure of the problem in an efficient manner. This is most obvious for large $K$ and $L$, e.g., \cref{fig:comparison}c.
% where $\rklucb$ is the only algorithm that finished exploration before the end of the trials. 
This happens because $\rklucb$ works with improved confidence intervals. It is worth noting that $\mu=\gamma$ for this problem, and hence $\mu^2 = \mu \gamma$, and according to \cref{thm:upper bound}, $\rklucb$ should not perform better than $\bilinucb$, yet it is $4$ times better as seen in \cref{fig:comparison}. This suggests that our upper bound is loose.
%Based on this result we also expect $\klucb$ to perform
% better than $\ucb$ on these instances, though its regret would still scale quadratically with $K=L$.
%Thus, although both $\bilinucb$ and $\rklucb$ leverage the problem structure, $\rklucb$ is superior because it exploits the distribution information.

\subsection{Models based on Real-World Data}
\label{sec:realworld}

In this experiment, we compare the performance of $\rklucb$ and other algorithms on models derived from the \emph{Yandex} dataset \cite{yandex}, an anonymized search log of $35$M search sessions. Each session contains a query, the list of displayed documents at positions $1$ to $10$, and the clicks on those documents. We extract the $20$ most frequent queries from the dataset, and estimate the parameters of the PBM model using the EM algorithm \cite{pyclick,chuklin15click}. 

In order to illustrate the typical models we obtain, we plot the learned parameters of two queries, Queries $1$ and $2$. \cref{fig:samplequeries}a shows the sorted attraction probabilities of items in the queries, and \cref{fig:samplequeries}b shows the sorted examination probabilities of the positions. Query $1$ has $L = 871$ items and Query $2$ has $L = 807$ items. We illustrate the performance on these queries because they differ notably in their $\mu$ \eqref{eq:average reward} and $\pmx$ \eqref{eq:maximum reward}, so we can study the performance of our algorithm in different real-world settings. \cref{fig:samplequeries}c and d show the regret of all algorithms on Queries $1$ and $2$, respectively. 

For Query $1$, $\rklucb$ is significantly better than $\bilinucb$ and $\ucbelim$, and no worse than $\ucb$. For Query $2$, $\rklucb$ is superior to all algorithms. Note that $\pmx = 0.85$ in Query $1$ is higher than $\pmx = 0.66$ in Query $2$. Also, $\mu = 0.13$ in Query $1$ is lower than $\mu = 0.28$ in Query $2$. From \cref{eq:gamma}, $\gamma=0.15$ for Query $1$, which is lower than $\gamma=0.34$ for Query $2$. Our upper bound (\cref{thm:upper bound}) on the regret of $\rklucb$ scales as $\mathcal{O}((\mu\gamma)^{-1})$, and consequently we expect $\rklucb$ to perform better on Query $2$.  The results % (shown on subfigures \textbf{c.} and \textbf{d.}) 
confirm this expectation.

\cref{fig:realworld} shows the regret averaged over all $20$ queries. Here we compute the average regret on the $20$ queries, and calculate the standard error over $5$ runs. $\rklucb$ has the lowest regret among all the algorithms; its regret is $10.9$ percent lower than that of $\ucb$, and $79$ percent lower than that of $\bilinucb$. This is expected: Some real-world instances have a benign rank-$1$ structure like Query $2$, while others do not, like Query $1$. Hence we see a reduction in the average gains of $\rklucb$ over $\ucb$ in \cref{fig:realworld} as compared to \cref{fig:samplequeries}d. The high regret of $\bilinucb$, which also is designed to exploit the problem structure, shows that it fails when faced with such unfavorable rank-$1$ problems. The fact that $\rklucb$ performs on-par with optimal algorithms on the hard problems, and is able to better leverage the problem structure on easy ones, makes it an appealing solution for practice.

\begin{figure}[h]
  \centering
  \includegraphics[width=1.8in]{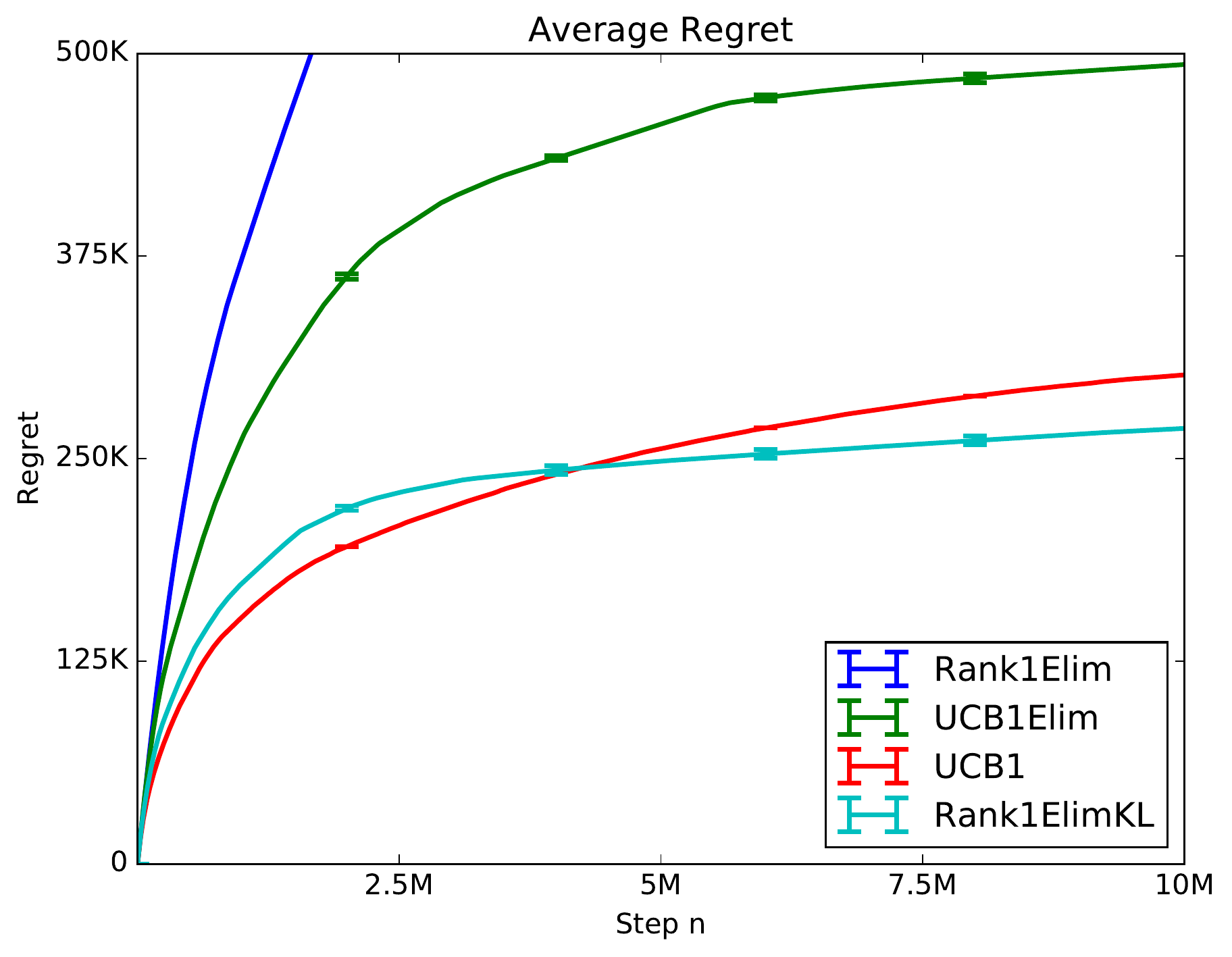} \vspace{-0.2in}
  \caption{The average $n$-step regret over all $20$ queries from the Yandex dataset, with $5$ runs per query.}
  \label{fig:realworld}
\end{figure}

%!TEX root = Paper.tex

\section{Related Work}
\label{sec:related work}

Our algorithm is based on the $\bilinucb$ algorithm of  Katariya \etal~\shortcite{rank1stochastic}; the main difference being
that we replace the confidence intervals of $\bilinucb$ that are based on subgaussian tail inequalities with confidence intervals based on KL divergences. As discussed beforehand, this result in an unilateral improvement of their regret bound:
The new algorithm is still able to exploit the problem structure of benign instances, while, unlike for $\bilinucb$, its regret
is still controlled even on instances that are ``hard'' for $\bilinucb$. 
As demonstrated in the previous section, the new algorithm is also a major practical improvement over $\bilinucb$, while staying competitive with alternatives on hard instances.
%$\rklucb$ solves \emph{Bernoulli rank-$1$ bandit}, which is a specialization of stochastic rank-$1$ bandit \cite{rank1stochastic} to Bernoulli rewards. Katariya \etal~\shortcite{rank1stochastic} propose an elimination algorithm, $\bilinucb$, when the rewards are sub-gaussian, and provide a regret bound of $\mathcal{O}\left(  \mu^{-2} \Delta^{-1} \log n \right)$. Our algorithm, $\rklucb$, is an algorithm for Bernoulli rewards, and its $n$-step regret is $O( (\mu\gamma\Delta)^{-1} \log n)$. This improvement is from $1/\mu$ to $1/\gamma$ is significant for real-world applications as we discuss in \cref{sec:discussion}, because click data often has $\mu=1/\max\{K,L\}$ and $\pmx \ll 1$, so that $\gamma = \max\set{\mu,1-\pmx} \gg \mu$. Our experiments in \cref{sec:experiments} show that the gains of $\rklucb$ over $\bilinucb$ are substantial.

Several other papers studied bandits where the payoff is given by a low rank matrix. Zhao \etal~\shortcite{zhao13interactive} proposed a bandit algorithm for low-rank matrix completion, which approximates the posterior of latent item features by a single point. The authors do not analyze this algorithm. Kawale \etal~\shortcite{kawale15efficient} proposed a bandit algorithm for low-rank matrix completion which uses Thompson sampling with Rao-Blackwellization. They analyze a variant of their algorithm whose $n$-step regret for rank-$1$ matrices is $O((1 / \Delta^2) \log n)$. This is suboptimal compared to our algorithm. Maillard \etal~\shortcite{maillard14latent} studied a multi-armed bandit problem where the arms are partitioned into several latent groups. In this work, we do not make any such assumptions, but our results are limited to rank $1$. Gentile \etal~\shortcite{gentile14online} proposed an algorithm that clusters users based on their preferences, under the assumption that the features of items are known. Sen \etal~\shortcite{sen2016contextual} proposed an algorithm for contextual bandits with latent confounders, which reduces to a multi-armed bandit problem where the reward matrix is low-rank. They use an NMF-based approach and require that the reward matrix obeys a variant of the restricted isometry property. We make no such assumptions. Our work also differs from all above papers in the setting. The learning agents controls both the choice of the row and column. In the above papers, the rows are controlled by the environment.

$\rklucb$ is motivated by the structure of the PBM \cite{richardson07predicting}. Lagree \etal~\shortcite{lagree16multipleplay} proposed a bandit algorithm for this model but they assume that the examination probabilities are known. $\rklucb$ can be used to solve this problem without this assumption. The cascade model \cite{craswell08experimental} is an alternative way of explaining the position bias in click data \cite{chuklin15click}. Bandit algorithms for this class of models have been proposed in several recent papers \cite{kveton15cascading,combes15learning,kveton15combinatorial,katariya16dcm,zong16cascading,li16contextual}. 
\vspace{-.05in}

%!TEX root = Paper.tex

\section{Conclusions}

In this work, we proposed $\rklucb$, an elimination based algorithm that uses $\klucb$ confidence intervals to find the maximum entry of a stochastic rank-$1$ matrix with Bernoulli rewards. 
The algorithm is a modification of the $\bilinucb$ algorithm \cite{rank1stochastic} where the subgaussian-type confidence intervals are replaced by ones that use KL divergences. 
As we demonstrate both empirically and analytically, this change results in a significant improvement.
%The regret of $\rklucb$ scales with $(\mu\gamma)^{-1}$ as opposed to $\mu^{-2}$ where $\mu$ is 
%is the minimum of the average row and column rewards where $\gamma\ge \mu$ and typically $\gamma\gg \mu$.
As a result, we obtain the first algorithm that is able to exploit the rank-1 structure without paying a significant penalty 
on instances where the rank-1 structure cannot be exploited.
\if0

We analyze its performance, and show an improved bound on its regret that scales as $1 / \mu$.
 \todob{Improved over what? Add a reference to the definition of $\mu$, at the minimum.} We also evaluate $\rklucb$ on synthetic as well as real-world problems, and show that its regret is lower than those of $\bilinucb$, $\ucbelim$, and $\ucb$. We especially highlight the performance improvement on the `needle in a haystack' problems, where the number of attractive items and examined positions is small. This is a good model of click data.

Although our analysis and experiments use Bernoulli click data, the algorithm generalizes to arbitrary bounded distributions, by replacing the Kullback-Leibler divergence appropriately. Since the $\klucb$ confidence intervals are tighter than those of $\ucb$ for bounded distributions, $\rklucb$ is guaranteed to perform better than $\bilinucb$ on these problems. Note that our analysis does not hold in this general setting, because \cref{lem:KL scaling} holds only for Bernoulli random variables.
\fi

Finally, we note that $\rklucb$ uses the rank-$1$ structure of the problem and there are no guarantees beyond rank $1$. While the dependence of the regret of $\rklucb$ on $1 / \Delta$ is known to be tight \cite{rank1stochastic}, the question about the optimal dependence on $1 / \mu$ is still open.

\if0
We note however that the problem is far from closed. Katariya \etal~\shortcite{rank1stochastic} prove a $\Omega( (K+L) \pmx^{-1} \Delta^{-1} \log n)$ lower bound for the Bernoulli rank-$1$ bandit, leaving room for improvement in the $\mathcal{O}(\mu^{-1} \gamma^{-1})$ dependence in either our algorithm or the lower bound.

 \cite{garivier11klucb}). 
 \fi

\bibliographystyle{named}
\bibliography{ijcai16}

%!TEX root = Paper.tex

\clearpage
\onecolumn
\appendix

\section{Proof of \cref{thm:upper bound}}
\label{sec:upper bound}
\newcommand{\eps}{\varepsilon}
We start by recalling Theorem 10 of Garivier and Cappe \shortcite{garivier11klucb} with a slight extension that follows immediately by inspecting their proof. We will comment on the difference after stating the definitions.
Let $(X_t)_{t\ge 1}$ be a sequence of random variables bounded in $[0,1]$. Assume that $(\cF_t)_{t\ge 1}$ is a filtration ($\cF_t\subset \cF_{t+1}$ are $\sigma$-algebras) and $(X_t)_{t\ge 1}$ is $(\cF_t)_t$-adapted (i.e., for $t\ge 1$, $X_1,\dots,X_t$ are $\cF_t$ measurable), and $\E{X_{t+1}|\cF_t}= \mu$ with some fixed value $\mu\in [0,1]$.
Let $(\eps_t)_{t\ge 1}$ be a sequence of $(\cF_t)$-previsible Bernoulli random variables: For all $t\ge 1$, $\eps_t$ is $\cF_{t-1}$-measurable with $\cF_0 = \cF$ the $\sigma$-algebra that holds all random variables. Define
\begin{align*}
S(t) &= \sum_{s=1}^t \eps_s X_s\,, \quad N(t) = \sum_{s=1}^t \eps_s\,,\quad \hat{\mu}(t) = \frac{S(t)}{N(t)}\,, \quad t\ge 1\,.
\end{align*}
The difference to the assumptions used by Garivier and Cappe \shortcite{garivier11klucb} is that they assume that 
the random variables $(X_t)_{t\ge 1}$ are independent with common mean $\mu$ and that for $s>t$, $X_{s}$ is independent of $\cF_t$. With this we are ready to state their theorem:
\begin{theorem}[After Theorem 10 of Garivier and Cappe \shortcite{garivier11klucb}]
\label{theorem:gc}
Let $(\hat{\mu}(t))_{t\ge 1}$ be as above and let
\begin{align*}
U(t) =\sup \{\, q>\hat{\mu}(t)\,: \, N(t) \,d(\hat{\mu}(t),q) \le \delta\, \}\,.
\end{align*}
Then,
\begin{align*}
P( U(t) <\mu ) \le e \, \lceil \delta \log(t) \rceil \exp(-\delta)\,.
\end{align*}
\end{theorem}

%\begin{theorem}
%\label{thm:upper bound} The expected $n$-step regret of $\rklucb$ is bounded as
%\begin{align*}
%  R(n) \leq
%  \frac{160}{\mu \gamma} \left(\sum_{i = 1}^K \frac{1}{\bar{\Delta}^\textsc{u}_i} +
%  \sum_{j = 1}^L \frac{1}{\bar{\Delta}^\textsc{v}_j}\right) \log n +
%  (K + L) (6 e + 82)
%\end{align*}
%for any $n \geq 5$, where
%\begin{align*}
%  \bar{\Delta}^\textsc{u}_i & = \Delta^\textsc{u}_i + \I{\Delta^\textsc{u}_i = 0} \Delta^\textsc{v}_{\min}\,, \\
%  \bar{\Delta}^\textsc{v}_j & = \Delta^\textsc{v}_j + \I{\Delta^\textsc{v}_j = 0} \Delta^\textsc{u}_{\min}\,.
%\end{align*}
%\end{theorem}

Let us now turn to our proof.
Let $\rnd{R}^\textsc{u}_\ell(i)$ be the stochastic regret associated with row $i$ in row exploration stage $\ell$ and $\rnd{R}^\textsc{v}_\ell(j)$ be the stochastic regret associated with column $j$ in column exploration stage $\ell$. Then the expected $n$-step regret of $\rklucb$ can be written as
\begin{align*}
  R(n) \leq
  \E{\sum_{\ell = 0}^{n - 1} \left(\sum_{i = 1}^K \rnd{R}^\textsc{u}_\ell(i) +
  \sum_{j = 1}^L \rnd{R}^\textsc{v}_\ell(j)\right)}\,,
\end{align*}
where the outer sum is over possibly $n$ stages. Let
\begin{align*}
  \cE^\textsc{u}_\ell =
  \{\text{Event $1$: } & \forall i \in \rnd{I}_\ell:
  \bar{\rnd{u}}_\ell(i) \in [\rnd{L}^\textsc{u}_\ell(i), \rnd{U}^\textsc{u}_\ell(i)]\,, \\
  \text{Event $2$: } & \forall i \in \rnd{I}_\ell: \bar{\rnd{u}}_\ell(i) \geq \mu \bar{u}(i)\,, \\
  \text{Event $3$: } & \forall i \in \rnd{I}_\ell \setminus \set{i^\ast}:
  n_\ell \geq \frac{16}{\mu \gamma (\Delta^\textsc{u}_i)^2} \log n \implies
  \hat{\rnd{u}}(i) \leq \rnd{c}_\ell [\bar{u}(i) + \Delta^\textsc{u}_i / 4]\,, \\
  \text{Event $4$: } & \forall i \in \rnd{I}_\ell \setminus \set{i^\ast}:
  n_\ell \geq \frac{16}{\mu \gamma (\Delta^\textsc{u}_i)^2} \log n \implies
  \hat{\rnd{u}}(i^\ast) \geq \rnd{c}_\ell [\bar{u}(i^\ast) - \Delta^\textsc{u}_i / 4]\}
\end{align*}
be ``good events'' associated with row $i$ at the end of stage $\ell$, where
\begin{align*}
  \bar{\rnd{u}}_\ell(i) =
  \sum_{t = 0}^\ell \condE{\sum_{j = 1}^L
  \frac{\rnd{C}^\textsc{u}_t(i, j) - \rnd{C}^\textsc{u}_{t - 1}(i, j)}{n_\ell}}{\rnd{h}^\textsc{v}_t} =
  \underbrace{\left(\sum_{t = 0}^\ell \frac{n_t - n_{t - 1}}{n_\ell}
  \sum_{j = 1}^L \frac{\bar{v}(\rnd{h}^\textsc{v}_t(j))}{L}\right)}_{\rnd{c}_\ell} \bar{u}(i)
\end{align*}
is the expected reward of row $i$ conditioned on column elimination strategy $\rnd{h}^\textsc{v}_0, \dots, \rnd{h}^\textsc{v}_\ell$; $\rnd{C}^\textsc{u}_{-1}(i, j) = 0$; and $n_{-1} = 0$. 
Let $\overline{\cE^\textsc{u}_\ell}$ be the complement of event $\cE^\textsc{u}_\ell$. 
Let
\begin{align*}
  \cE^\textsc{v}_\ell =
  \{\text{Event $1$: } & \forall j \in \rnd{J}_\ell:
  \bar{\rnd{v}}_\ell(j) \in [\rnd{L}^\textsc{v}_\ell(j), \rnd{U}^\textsc{v}_\ell(j)]\,, \\
  \text{Event $2$: } & \forall j \in \rnd{J}_\ell: \bar{\rnd{v}}_\ell(j) \geq \mu \bar{v}(j)\,, \\
  \text{Event $3$: } & \forall j \in \rnd{J}_\ell \setminus \set{j^\ast}:
  n_\ell \geq \frac{16}{\mu \gamma (\Delta^\textsc{v}_j)^2} \log n \implies
  \hat{\rnd{v}}(j) \leq \rnd{c}_\ell [\bar{v}(j) + \Delta^\textsc{v}_j / 4]\,, \\
  \text{Event $4$: } & \forall j \in \rnd{J}_\ell \setminus \set{j^\ast}:
  n_\ell \geq \frac{16}{\mu \gamma (\Delta^\textsc{v}_j)^2} \log n \implies
  \hat{\rnd{v}}(j^\ast) \geq \rnd{c}_\ell [\bar{v}(j^\ast) - \Delta^\textsc{v}_j / 4]\}
\end{align*}
be ``good events'' associated with column $j$ at the end of stage $\ell$, where
\begin{align*}
  \bar{\rnd{v}}_\ell(j) =
  \sum_{t = 0}^\ell \condE{\sum_{i = 1}^K
  \frac{\rnd{C}^\textsc{v}_t(i, j) - \rnd{C}^\textsc{v}_{t - 1}(i, j)}{n_\ell}}{\rnd{h}^\textsc{u}_t} =
  \underbrace{\left(\sum_{t = 0}^\ell \frac{n_t - n_{t - 1}}{n_\ell}
  \sum_{i = 1}^K \frac{\bar{u}(\rnd{h}^\textsc{u}_t(i))}{K}\right)}_{\rnd{c}_\ell} \bar{v}(j)
\end{align*}
is the expected reward of column $j$ conditioned on row elimination strategy $\rnd{h}^\textsc{u}_0, \dots, \rnd{h}^\textsc{u}_\ell$; $\rnd{C}^\textsc{v}_{-1}(i, j) = 0$; and $n_{-1} = 0$. Let $\overline{\cE^\textsc{v}_\ell}$ be the complement of event $\cE^\textsc{v}_\ell$. Let $\cE$ be the event that all events $\cE^\textsc{u}_\ell$ and $\cE^\textsc{v}_\ell$ happen; and $\ccE$ be the complement of $\cE$, the event that at least one of $\cE^\textsc{u}_\ell$ and $\cE^\textsc{v}_\ell$ does not happen. Then the expected $n$-step regret can be bounded from above as
\begin{align*}
  R(n)
  & \leq \E{\left(\sum_{\ell = 0}^{n - 1} \left(\sum_{i = 1}^K \rnd{R}^\textsc{u}_\ell(i) +
  \sum_{j = 1}^L \rnd{R}^\textsc{v}_\ell(j)\right)\right) \I{\cE}} +
  n P(\ccE) \\
  & \leq \E{\left(\sum_{\ell = 0}^{n - 1} \left(\sum_{i = 1}^K \rnd{R}^\textsc{u}_\ell(i) +
  \sum_{j = 1}^L \rnd{R}^\textsc{v}_\ell(j)\right)\right) \I{\cE}} +
  (K + L) (6 e + 2) \\
  & = \sum_{i = 1}^K \E{\sum_{\ell = 0}^{n - 1} \rnd{R}^\textsc{u}_\ell(i) \I{\cE}} +
  \sum_{j = 1}^L \E{\sum_{\ell = 0}^{n - 1} \rnd{R}^\textsc{v}_\ell(j) \I{\cE}} +
  (K + L) (6 e + 2)\,,
\end{align*}
where the second inequality is from \cref{lem:bad events}.

Let $\cH_\ell = (\rnd{I}_\ell, \rnd{J}_\ell)$ be the rows and columns in stage $\ell$, and
\begin{align*}
  \mathcal{F}_\ell =
  \set{\forall i \in \rnd{I}_\ell: \sqrt{\mu \gamma} \Delta^\textsc{u}_i \leq \tilde{\Delta}_{\ell - 1}, \
  \forall j \in \rnd{J}_\ell: \sqrt{\mu \gamma} \Delta^\textsc{v}_j \leq \tilde{\Delta}_{\ell - 1}}
\end{align*}
be the event that all rows and columns with ``large gaps'' are eliminated by the beginning of stage $\ell$. By \cref{lem:maximum elimination stage}, event $\mathcal{F}_\ell$ happens when event $\cE$ happens. Moreover, the expected regret in stage $\ell$ is independent of $\mathcal{F}_\ell$ given $\cH_\ell$. Therefore, we can bound the regret from above as
\begin{align}
  R(n) \leq
  \sum_{i = 1}^K \E{\sum_{\ell = 0}^{n - 1} \condE{\rnd{R}^\textsc{u}_\ell(i)}{\cH_\ell} \I{\cF_\ell}} +
  \sum_{j = 1}^L \E{\sum_{\ell = 0}^{n - 1} \condE{\rnd{R}^\textsc{v}_\ell(j)}{\cH_\ell} \I{\cF_\ell}} +
  (K + L) (6 e + 2)\,.
  \label{eq:component regret}
\end{align}
By \cref{lem:row regret},
\begin{align*}
  \E{\sum_{\ell = 0}^{n - 1} \condE{\rnd{R}^\textsc{u}_\ell(i)}{\cH_\ell} \I{\cF_\ell}}
  & \leq \frac{160}{\mu \gamma \bar{\Delta}^\textsc{u}_i} \log n + 80\,, \\
  \E{\sum_{\ell = 0}^{n - 1} \condE{\rnd{R}^\textsc{v}_\ell(j)}{\cH_\ell} \I{\cF_\ell}}
  & \leq \frac{160}{\mu \gamma \bar{\Delta}^\textsc{v}_j} \log n + 80\,.
\end{align*}
Now we apply the above upper bounds to \eqref{eq:component regret} and get our main claim.

\section{Technical Lemmas}
\label{sec:lemmas}

\begin{lemma}
\label{lem:bad events} Let $\ccE$ be defined as in the proof of \cref{thm:upper bound}. Then for any $n \geq 5$,
\begin{align*}
  P(\ccE) \leq
  \frac{(K + L) (6 e + 2)}{n}\,.
\end{align*}
\end{lemma}
\begin{proof}
\newcommand{\Eu}[1]{\cE^{\textsc{u}}_{#1}}
\newcommand{\Ev}[1]{\cE^{\textsc{v}}_{#1}}
\newcommand{\Ea}[1]{\cE_{#1}}
\newcommand{\cEu}[1]{\overline{\cE^{\textsc{u}}_{#1}}}
\newcommand{\cEv}[1]{\overline{\cE^{\textsc{v}}_{#1}}}
\newcommand{\cEa}[1]{\overline{\cE_{#1}}}
Let $\Ea{\ell} = \Eu{\ell}\cap \Ev{\ell}$.
Then, $\ccE = \cEa{0} \cup (\cEa{1} \cap \Ea{0}) \cup \dots \cup (\cEa{n-1}\cap \Ea{0}\cap \dots \cap \Ea{n-2})$.
By the same logic, $\cEa{\ell} \cap \Ea{0} \cap \dots \cap \Ea{\ell-1} = (\cEu{\ell} \cap \Ea{0} \cap \dots \cap \Ea{\ell-1})
\cup (\cEv{\ell} \cap \Eu{\ell} \cap \Ea{0} \cap \dots \cap \Ea{\ell-1})$.
Hence,
\begin{align*}
P(\ccE) \leq
\sum_{\ell=0}^{n-1}
P(\cEu{\ell}, \Ea{0}, \dots, \Ea{\ell-1})
+P(\cEv{\ell}, \Ea{0} ,\dots, \Ea{\ell-1})\,.
\end{align*}
\if0
By the chain rule and from the rules of probability,
\begin{align*}
  P(\ccE)
  & = \sum_{\ell = 0}^{n - 1} P(\overline{\cE^\textsc{u}_\ell}, \overline{\cE^\textsc{v}_\ell} \mid
  \cE^\textsc{u}_0, \ \dots, \ \cE^\textsc{u}_{\ell - 1}, \cE^\textsc{v}_0, \ \dots, \ \cE^\textsc{v}_{\ell - 1}) \\
  & \leq \sum_{\ell = 0}^{n - 1} P(\overline{\cE^\textsc{u}_\ell} \mid
  \cE^\textsc{u}_0, \ \dots, \ \cE^\textsc{u}_{\ell - 1}, \cE^\textsc{v}_0, \ \dots, \ \cE^\textsc{v}_{\ell - 1}) +
  \sum_{\ell = 0}^{n - 1} P(\overline{\cE^\textsc{v}_\ell} \mid \overline{\cE^\textsc{u}_\ell},
  \cE^\textsc{u}_0, \ \dots, \ \cE^\textsc{u}_{\ell - 1}, \cE^\textsc{v}_0, \ \dots, \ \cE^\textsc{v}_{\ell - 1}) \\
  & \leq \sum_{\ell = 0}^{n - 1} P(\overline{\cE^\textsc{u}_\ell} \mid
  \cE^\textsc{u}_0, \ \dots, \ \cE^\textsc{u}_{\ell - 1}, \cE^\textsc{v}_0, \ \dots, \ \cE^\textsc{v}_{\ell - 1}) +
  \sum_{\ell = 0}^{n - 1} P(\overline{\cE^\textsc{v}_\ell} \mid
  \cE^\textsc{u}_0, \ \dots, \ \cE^\textsc{u}_{\ell - 1}, \cE^\textsc{v}_0, \ \dots, \ \cE^\textsc{v}_{\ell - 1})\,.
\end{align*}
\fi
Now we bound the probability of the events $\overline{\cE^\textsc{u}_\ell}, \cE^\textsc{u}_0, \ \dots, \ \cE^\textsc{u}_{\ell - 1}, \cE^\textsc{v}_0, \ \dots, \ \cE^\textsc{v}_{\ell - 1}$; and then sum them up. The proof for the probability of the second term above is analogous and hence it is omitted.

\vspace{0.1in}

\noindent \textbf{Event $1$}

\noindent The probability that event $1$ in $\cE^\textsc{u}_\ell$ does not happen is bounded as follows. For any $i \in [K]$ and $\rnd{h}^\textsc{v}_0, \dots, \rnd{h}^\textsc{v}_\ell$,
\begin{align*}
  P(\bar{\rnd{u}}_\ell(i) \notin [\rnd{L}^\textsc{u}_\ell(i), \rnd{U}^\textsc{u}_\ell(i)])
  & \leq P(\bar{\rnd{u}}_\ell(i) < \rnd{L}^\textsc{u}_\ell(i)) + P(\bar{\rnd{u}}_\ell(i) > \rnd{U}^\textsc{u}_\ell(i)) \\
  & \leq \frac{2 e \ceils{\log(n \log^3 n) \log n_\ell}}{n \log^3 n} \\
  & \leq \frac{2 e \ceils{\log^2 n + \log(\log^3 n) \log n}}{n \log^3 n} \\
  & \leq \frac{2 e \ceils{2 \log^2 n}}{n \log n} \\
  & \leq \frac{6 e}{n \log n}\,,
\end{align*}
where the second inequality is from \cref{theorem:gc}, \todoc{Clearly, we need to extend their result because in our case
the conditional means shift over time.}
the third inequality is from $n \geq n_\ell$, the fourth inequality is from $\log(\log^3 n) \leq \log n$ for $n \geq 5$, and the last inequality is from $\ceils{2 \log^2 n} \leq 3 \log^2 n$ for $n \geq 3$. By the union bound,
\begin{align*}
  P(\exists i \in \rnd{I}_\ell \text{ s.t. } \bar{\rnd{u}}_\ell(i) \notin [\rnd{L}^\textsc{u}_\ell(i), \rnd{U}^\textsc{u}_\ell(i)]) \leq
  \frac{6 e K}{n \log n}
\end{align*}
for any $\rnd{I}_\ell$ and $\rnd{h}^\textsc{v}_0, \dots, \rnd{h}^\textsc{v}_\ell$. Finally, we take the expectation over $\rnd{I}_\ell$ and $\rnd{h}^\textsc{v}_0, \dots, \rnd{h}^\textsc{v}_\ell$; and have that the probability that event $1$ in $\cE^\textsc{u}_\ell$ does not happen at the end of stage $\ell$ is bounded as above.

\vspace{0.1in}

\noindent \textbf{Event $2$}

\noindent Event $2$ in $\cE^\textsc{u}_\ell$ is guaranteed to happen, $\bar{\rnd{u}}_\ell(i) \geq \mu \bar{u}(i)$ for all $i \in \rnd{I}_\ell$. This claim holds trivially when $\ell = 0$, because all columns in row elimination stage $0$ are chosen with the same probability. When $\ell > 0$, all column confidence intervals up to stage $\ell$ hold because events $\cE^\textsc{v}_0, \dots, \cE^\textsc{v}_{\ell - 1}$ happen. Therefore, by the design of $\rklucb$, any eliminated column $j$ up to stage $\ell$ is substituted with column $j'$ such that $\bar{v}(j') \geq \allowbreak \bar{v}(j)$. Since the columns in any row elimination stage are chosen randomly, $\bar{\rnd{u}}_\ell(i) \geq \mu \bar{u}(i)$ for all $i \in \rnd{I}_\ell$.

\vspace{0.1in}

\noindent \textbf{Event $3$}

\noindent The probability that event $3$ in $\cE^\textsc{u}_\ell$ does not happen is bounded as follows. If the event does not happen in row $i$, then
\begin{align*}
  n_\ell \geq \frac{16}{\mu \gamma (\Delta^\textsc{u}_i)^2} \log n\,, \quad
  \hat{\rnd{u}}(i) > \rnd{c}_\ell [\bar{u}(i) + \Delta^\textsc{u}_i / 4]\,.
\end{align*}
From Hoeffding's inequality and $\E{\hat{\rnd{u}}(i)} = \rnd{c}_\ell \bar{u}(i)$, we have that
\begin{align*}
  P(\hat{\rnd{u}}(i) > \rnd{c}_\ell [\bar{u}(i) + \Delta^\textsc{u}_i / 4]) \leq
  \exp[- n_\ell d(\rnd{c}_\ell [\bar{u}(i) + \Delta^\textsc{u}_i / 4], \rnd{c}_\ell \bar{u}(i))]\,.
\end{align*}
From our scaling lemma (\cref{lem:KL scaling}), the inequality $\rnd{c}_\ell \geq \mu$ and the definition 
$\gamma = \max(\mu,1 - \pmx)$,
we have that
\begin{align*}
  \exp[- n_\ell d(\rnd{c}_\ell [\bar{u}(i) + \Delta^\textsc{u}_i / 4], \rnd{c}_\ell \bar{u}(i))]
%  & \leq \exp[- n_\ell \rnd{c}_\ell (1 - \pmx) d(\bar{u}(i) + \Delta^\textsc{u}_i / 4, \bar{u}(i))] \\
  & \leq \exp[- n_\ell \,\mu \gamma\, (\Delta^\textsc{u}_i)^2 / 8]\,.
\end{align*}
Finally, from our assumption on $n_\ell$, we conclude that
\begin{align*}
  \exp[- n_\ell \mu \gamma (\Delta^\textsc{u}_i)^2 / 8] \leq
  \exp[- 2 \log n] =
  \frac{1}{n^2}\,.
\end{align*}
Now we chain all inequalities and observe that event $3$ in $\cE^\textsc{u}_\ell$ does not happen with probability of at most $K / n^2$ for any $\rnd{I}_\ell$ and $\rnd{h}^\textsc{v}_0, \dots, \rnd{h}^\textsc{v}_\ell$. Finally, we take the expectation over $\rnd{I}_\ell$ and $\rnd{h}^\textsc{v}_0, \dots, \rnd{h}^\textsc{v}_\ell$; and have that the probability that event $3$ in $\cE^\textsc{u}_\ell$ does not happen at the end of stage $\ell$ is at most $K / n^2$.

\vspace{0.1in}

\noindent \textbf{Event $4$}

\noindent The probability that event $4$ in $\cE^\textsc{u}_\ell$ does not happen can be bounded similarly to that of event $3$. If the event does not happen in row $i$, then
\begin{align*}
  n_\ell \geq \frac{16}{\mu \gamma (\Delta^\textsc{u}_i)^2} \log n\,, \quad
  \hat{\rnd{u}}(i^\ast) < \rnd{c}_\ell [\bar{u}(i^\ast) - \Delta^\textsc{u}_i / 4]\,.
\end{align*}
Then by the same reasoning as in event $3$,
\begin{align*}
  P(\hat{\rnd{u}}(i^\ast) < \rnd{c}_\ell [\bar{u}(i^\ast) - \Delta^\textsc{u}_i / 4])
  & \leq \exp[- n_\ell d(\rnd{c}_\ell [\bar{u}(i^\ast) - \Delta^\textsc{u}_i / 4], \rnd{c}_\ell \bar{u}(i^\ast))] \\
  & \leq \exp[- n_\ell \mu \gamma (\Delta^\textsc{u}_i)^2 / 8] \\
  & \leq \exp[- 2 \log n] \\
  & = \frac{1}{n^2}\,.
\end{align*}
This implies that event $4$ in $\cE^\textsc{u}_\ell$ does not happen with probability of at most $K / n^2$ for any $\rnd{I}_\ell$ and $\rnd{h}^\textsc{v}_0, \dots, \rnd{h}^\textsc{v}_\ell$. Finally, we take the expectation over $\rnd{I}_\ell$ and $\rnd{h}^\textsc{v}_0, \dots, \rnd{h}^\textsc{v}_\ell$; and have that the probability that event $4$ in $\cE^\textsc{u}_\ell$ does not happen at the end of stage $\ell$ is at most $K / n^2$\,.

\vspace{0.1in}

\noindent \textbf{Total probability}

\noindent Note that the maximum number of stages in $\rklucb$ is $\log n$. By the union bound, we get that
\begin{align*}
  P(\ccE)
  & \leq \left(\frac{6 e K}{n \log n} + \frac{K}{n^2} + \frac{K}{n^2}\right) \log n +
  \left(\frac{6 e L}{n \log n} + \frac{L}{n^2} + \frac{L}{n^2}\right) \log n \\
  & \leq \frac{(K + L) (6 e + 2)}{n}\,.
\end{align*}
This concludes our proof.
\end{proof}

\begin{lemma}
\label{lem:maximum elimination stage} Let event $\cE$ happen and $m$ be the first stage where $\tilde{\Delta}_m < \sqrt{\mu \gamma} \Delta^\textsc{u}_i$. Then row $i$ must be eliminated by the end of stage $m$. Moreover, let $m$ be the first stage where $\tilde{\Delta}_m < \sqrt{\mu \gamma} \Delta^\textsc{v}_j$. Then column $j$ must be eliminated by the end of stage $m$.
\end{lemma}
\begin{proof}
We only prove the first claim. The other claim is proved analogously.

From the definition of $n_m$ and our assumption on $\tilde{\Delta}_m$,
\begin{align}
  n_m \geq
  \frac{16}{\tilde{\Delta}_m^2} \log n >
  \frac{16}{\mu \gamma (\Delta^\textsc{u}_i)^2} \log n\,.
  \label{eq:contradiction}
\end{align}
Suppose that $\rnd{U}^\textsc{u}_m(i) \geq \rnd{c}_m [\bar{u}(i) + \Delta^\textsc{u}_i / 2]$ happens. Then from this assumption, the definition of $\rnd{U}^\textsc{u}_m(i)$, and event $3$ in $\cE^\textsc{u}_m$,
\begin{align*}
  d(\hat{\rnd{u}}(i), \rnd{U}^\textsc{u}_m(i))
  & \geq d^+(\hat{\rnd{u}}(i), \rnd{c}_m [\bar{u}(i) + \Delta^\textsc{u}_i / 2]) \\
  & \geq d(\rnd{c}_m [\bar{u}(i) + \Delta^\textsc{u}_i / 4], \rnd{c}_m [\bar{u}(i) + \Delta^\textsc{u}_i / 2])\,,
\end{align*}
where $d^+(p, q) = d(p, q) \I{p \leq q}$.
From our scaling lemma (\cref{lem:KL scaling}), the inequality $\rnd{c}_\ell \geq \mu$ and the definition 
$\gamma = \max(\mu,1 - \pmx)$,
we further have that
\begin{align*}
  d(\rnd{c}_m [\bar{u}(i) + \Delta^\textsc{u}_i / 4], \rnd{c}_m [\bar{u}(i) + \Delta^\textsc{u}_i / 2])
%  & \geq \rnd{c}_m (1 - \pmx) d(\bar{u}(i) + \Delta^\textsc{u}_i / 4, \bar{u}(i) + \Delta^\textsc{u}_i / 2) \\
  & \geq \frac{\mu \gamma\, (\Delta^\textsc{u}_i)^2}{8}\,.
\end{align*}
From the definition of $\rnd{U}^\textsc{u}_m(i)$ and above inequalities,
\begin{align*}
  n_m =
  \frac{2 \log n}{d(\hat{\rnd{u}}(i), \rnd{U}^\textsc{u}_m(i))} \leq
  \frac{16 \log n}{\mu \gamma (\Delta^\textsc{u}_i)^2}\,.
\end{align*}
This contradicts to \eqref{eq:contradiction}, and therefore it must be true that $\rnd{U}^\textsc{u}_m(i) < \rnd{c}_m [\bar{u}(i) + \Delta^\textsc{u}_i / 2]$.

Now suppose that $\rnd{L}^\textsc{u}_m(i^\ast) \leq \rnd{c}_m [\bar{u}(i^\ast) - \Delta^\textsc{u}_i / 2]$ happens. Then from this assumption, the definition of $\rnd{L}^\textsc{u}_m(i^\ast)$, and event $4$ in $\cE^\textsc{u}_m$,
\begin{align*}
  d(\hat{\rnd{u}}(i^\ast), \rnd{L}^\textsc{u}_m(i^\ast))
  & \geq d^-(\hat{\rnd{u}}(i^\ast), \rnd{c}_m [\bar{u}(i^\ast) - \Delta^\textsc{u}_i / 2]) \\
  & \geq d(\rnd{c}_m [\bar{u}(i^\ast) - \Delta^\textsc{u}_i / 4], \rnd{c}_m [\bar{u}(i^\ast) - \Delta^\textsc{u}_i / 2])\,,
\end{align*}
where $d^-(p, q) = d(p, q) \I{p \geq q}$. 
From our scaling lemma (\cref{lem:KL scaling}), the inequality $\rnd{c}_\ell \geq \mu$ and the definition 
$\gamma = \max(\mu,1 - \pmx)$,
we further have that
\begin{align*}
  d(\rnd{c}_m [\bar{u}(i^\ast) - \Delta^\textsc{u}_i / 4], \rnd{c}_m [\bar{u}(i^\ast) - \Delta^\textsc{u}_i / 2])
%  & \geq \rnd{c}_m \gamma d(\bar{u}(i^\ast) - \Delta^\textsc{u}_i / 4, \bar{u}(i^\ast) - \Delta^\textsc{u}_i / 2) \\
  & \geq \frac{\mu\gamma\, (\Delta^\textsc{u}_i)^2}{8}\,.
\end{align*}
From the definition of $\rnd{L}^\textsc{u}_m(i^\ast)$ and above inequalities,
\begin{align*}
  n_m =
  \frac{2 \log n}{d(\hat{\rnd{u}}(i^\ast), \rnd{L}^\textsc{u}_m(i^\ast))} \leq
  \frac{16 \log n}{\mu \gamma (\Delta^\textsc{u}_i)^2}\,.
\end{align*}
This contradicts to \eqref{eq:contradiction}, and therefore it must be true that $\rnd{L}^\textsc{u}_m(i^\ast) > \rnd{c}_m [\bar{u}(i^\ast) - \Delta^\textsc{u}_i / 2]$.

Finally, it follows that row $i$ is eliminated by the end of stage $m$ because
\begin{align*}
  \rnd{U}^\textsc{u}_m(i) <
  \rnd{c}_m [\bar{u}(i) + \Delta^\textsc{u}_i / 2] =
  \rnd{c}_m [\bar{u}(i^\ast) - \Delta^\textsc{u}_i / 2] <
  \rnd{L}^\textsc{u}_m(i^\ast)\,.
\end{align*}
This concludes our proof.
\end{proof}

\begin{lemma}
\label{lem:row regret} The expected regret associated with any row $i \in [K]$ is bounded as
\begin{align*}
  \E{\sum_{\ell = 0}^{n - 1} \condE{\rnd{R}^\textsc{u}_\ell(i)}{\cH_\ell} \I{\cF_\ell}} \leq
  \frac{160}{\mu \gamma \bar{\Delta}^\textsc{u}_i} \log n + 80\,.
\end{align*}
Moreover, the expected regret associated with any column $j \in [L]$ is bounded as
\begin{align*}
  \E{\sum_{\ell = 0}^{n - 1} \condE{\rnd{R}^\textsc{v}_\ell(j)}{\cH_\ell} \I{\cF_\ell}} \leq
  \frac{160}{\mu \gamma \bar{\Delta}^\textsc{v}_j} \log n + 80\,.
\end{align*}
\end{lemma}
\begin{proof}
We only prove the first claim. The other claim is proved analogously.

This proof has two parts. In the first part, we assume that row $i$ is suboptimal. In the second part, we assume that row $i$ is optimal, $\Delta^\textsc{u}_i = 0$.

\vspace{0.1in}

\noindent \textbf{Row $i$ is suboptimal}

\noindent Let row $i$ be suboptimal and $m$ be the first stage where $\tilde{\Delta}_m < \sqrt{\mu \gamma} \Delta^\textsc{u}_i$. Then row $i$ is guaranteed to be eliminated by the end of stage $m$ (\cref{lem:maximum elimination stage}), and therefore
\begin{align*}
  \E{\sum_{\ell = 0}^{n - 1} \condE{\rnd{R}^\textsc{u}_\ell(i)}{\cH_\ell} \I{\cF_\ell}} \leq
  \E{\sum_{\ell = 0}^m \condE{\rnd{R}^\textsc{u}_\ell(i)}{\cH_\ell} \I{\cF_\ell}}\,.
\end{align*}
By Lemma 4 of Katariya \etal~\shortcite{rank1stochastic}, the expected regret of choosing row $i$ in stage $\ell$ can be bounded from above as
\begin{align*}
  \condE{\rnd{R}^\textsc{u}_\ell(i)}{\cH_\ell} \I{\cF_\ell} \leq
  (\Delta^\textsc{u}_i + 2^{m - \ell + 1} \Delta^\textsc{u}_i) (n_\ell - n_{\ell - 1})\,,
\end{align*}
where $n_\ell$ is the number of steps by the end of stage $\ell$, $2^{m - \ell + 1} \Delta^\textsc{u}_i$ is an upper bound on the gap of any non-eliminated column in stage $\ell \leq m$, and $n_{-1} = 0$. The bound follows from the observation that if column $j$ is not eliminated before stage $\ell$, then
\begin{align*}
  \Delta^\textsc{v}_j \leq
  \frac{\tilde{\Delta}_{\ell - 1}}{\sqrt{\mu \gamma}} =
  \frac{2^{m - \ell + 1} \tilde{\Delta}_m}{\sqrt{\mu \gamma}} <
  2^{m - \ell + 1} \Delta^\textsc{u}_i\,.
\end{align*}
It follows that
\begin{align*}
  \sum_{\ell = 0}^m (\Delta^\textsc{u}_i + 2^{m - \ell + 1} \Delta^\textsc{u}_i) (n_\ell - n_{\ell - 1})
  & \leq \Delta^\textsc{u}_i n_m + \Delta^\textsc{u}_i \sum_{\ell = 0}^m 2^{m - \ell + 1} n_\ell \\
  & \leq 2^4 \Delta^\textsc{u}_i (2^{2 m} \log n + 1) +
  2^4 \Delta^\textsc{u}_i \sum_{\ell = 0}^m 2^{m - \ell + 1} (2^{2 m} \log n + 1) \\
  & = 2^{2 m + 4} \Delta^\textsc{u}_i \log n + 16 \Delta^\textsc{u}_i +
  2^{2m + 6} \Delta^\textsc{u}_i \log n + 64 \Delta^\textsc{u}_i \\
  & \leq 5 \cdot 2^6 \cdot 2^{2 m - 2} \Delta^\textsc{u}_i \log n + 80\,.
\end{align*}
From the definition of $m$, we have that
\begin{align*}
  2^{m - 1} =
  \frac{1}{\tilde{\Delta}_{m - 1}} \leq
  \frac{1}{\sqrt{\mu \gamma} \Delta^\textsc{u}_i}\,.
\end{align*}
Now we chain all above inequalities and get that
\begin{align*}
  \E{\sum_{\ell = 0}^{n - 1} \condE{\rnd{R}^\textsc{u}_\ell(i)}{\cH_\ell} \I{\cF_\ell}}
  & \leq \sum_{\ell = 0}^m (\Delta^\textsc{u}_i + 2^{m - \ell + 1} \Delta^\textsc{u}_i) (n_\ell - n_{\ell - 1}) \\
  & \leq \frac{160}{\mu \gamma \Delta^\textsc{u}_i} \log n + 80\,.
\end{align*}
This concludes the first part of our proof.

\vspace{0.1in}

\noindent \textbf{Row $i$ is optimal}

\noindent Let row $i$ be optimal and $m$ be the first stage where $\tilde{\Delta}_m < \sqrt{\mu \gamma} \Delta^\textsc{v}_{\min}$. Then similarly to the first part of the analysis,
\begin{align*}
  \E{\sum_{\ell = 0}^{n - 1} \condE{\rnd{R}^\textsc{u}_\ell(i)}{\cH_\ell} \I{\cF_\ell}} \leq
  \frac{160}{\mu \gamma \Delta^\textsc{v}_{\min}} \log n + 80\,.
\end{align*}
This concludes our proof.
\end{proof}

\end{document}